\documentclass{article}

\usepackage{microtype}
\usepackage{graphicx}
\usepackage{subfigure}
\usepackage{booktabs} 

\usepackage{multirow}
\usepackage{makecell}

\usepackage{hyperref}

\usepackage[accepted]{icml2024}

\usepackage{amsmath}
\usepackage{amssymb}
\usepackage{mathtools}
\usepackage{amsthm}

\usepackage[capitalize,noabbrev]{cleveref}

\usepackage{authblk}
\usepackage{amsmath,amsfonts}
\usepackage{url}
\usepackage{bm}
\usepackage{comment}
\usepackage{bbm}
\usepackage{soul}
\usepackage{dsfont}
\usepackage{csquotes}
\usepackage{nicefrac}

\theoremstyle{plain}
\newtheorem{theorem}{Theorem}[section]

\newtheorem{lemma}[theorem]{Lemma}
\newtheorem{corollary}[theorem]{Corollary}
\theoremstyle{definition}
\newtheorem{definition}[theorem]{Definition}
\newtheorem{assumption}[theorem]{Assumption}
\theoremstyle{remark}
\newtheorem{remark}[theorem]{Remark}
\newtheorem{fact}[theorem]{Fact}

\usepackage[textsize=tiny]{todonotes}

\icmltitlerunning{Understanding the Training Speedup from Sampling with Approximate Losses}

\begin{document}

\twocolumn[
\icmltitle{Understanding the Training Speedup from Sampling with Approximate Losses}




\begin{icmlauthorlist}
\icmlauthor{Rudrajit Das}{ut}
\icmlauthor{Xi Chen}{am}
\icmlauthor{Bertram Ieong}{am}
\icmlauthor{Parikshit Bansal}{ut}
\icmlauthor{Sujay Sanghavi}{ut,am}
\end{icmlauthorlist}

\icmlaffiliation{ut}{UT Austin}
\icmlaffiliation{am}{Amazon}

\icmlcorrespondingauthor{Rudrajit Das}{rdas@utexas.edu}
\icmlcorrespondingauthor{Xi Chen}{xichex@amazon.com}
\icmlcorrespondingauthor{Sujay Sanghavi}{sanghavi@mail.utexas.edu}

\icmlkeywords{Machine Learning, ICML}

\vskip 0.3in
]



\printAffiliationsAndNotice{}  

\begin{abstract}
It is well known that selecting samples with large losses/gradients can significantly reduce the number of training steps. However, the selection overhead is often too high to yield any meaningful gains in terms of overall training time. In this work, we focus on the greedy approach of selecting samples with large \textit{approximate losses} instead of exact losses in order to reduce the selection overhead. For smooth convex losses, we show that such a greedy strategy can converge to a constant factor of the minimum value of the average loss in fewer iterations than the standard approach of random selection. We also theoretically quantify the effect of the approximation level.
We then develop SIFT which uses early exiting to obtain approximate losses with an intermediate layer's representations for sample selection. We evaluate SIFT on the task of training a 110M parameter 12 layer BERT base model, and show significant gains (in terms of training hours and number of backpropagation steps) without any optimized implementation over vanilla training. For e.g., to reach 64\% validation accuracy, SIFT with exit at the first layer takes $\sim$ 43 hours compared to $\sim$ 57 hours of vanilla training. 
\end{abstract}

\section{Introduction}
\label{sec:intro}
Stochastic Gradient Descent (SGD) and its variants are the algorithms of choice for solving large-scale optimization problems that arise in training machine learning models. These are problems of the form $\min_{\bm{w} \in \mathbb{R}^d} F(\bm{w})$, where $F(.)$ is the expected population loss and $\bm{w} \in \mathbb{R}^d$ is the vector of model parameters. More specifically, if $f(\bm{w}, .)$ is the per-sample loss and $\mathcal{D}$ is the data distribution, then $F(\bm{w}) = \mathbb{E}_{\bm{x} \sim \mathcal{D}}[f(\bm{w}, \bm{x})]$. The standard SGD update rule at $\bm{w}$ with step-size/learning rate $\eta$ is:
\begin{equation}
    \label{eq:1-sgd}
    \bm{w}_{+} = \bm{w} - \eta \nabla f(\bm{w}, \bm{x}),
\end{equation}
where the sample/example $\bm{x}$ is drawn from $\mathcal{D}$. In practice, the gradient over a single sample is replaced by the average gradient over a mini-batch of samples; for our theoretical results, we shall stick to batch-size = 1 (as in \cref{eq:1-sgd}).

The choice of the sample $\bm{x}$ in \cref{eq:1-sgd} significantly impacts the speed of convergence. 
There is copious work on selecting \enquote{important} samples for speeding up convergence. In this paradigm, the seminal idea is that of \textit{importance sampling} which proposes to sample the examples such that resulting stochastic gradient is unbiased while having the minimum possible variance \citep{zhao2015stochastic, alain2015variance, needell2014stochastic}. It turns out that the optimal solution is to sample the examples with probability proportional to their per-sample gradient norms. While this has a clean theoretical solution, it is completely infeasible from the practical standpoint because of the high cost of computing the per-sample gradient norms. To address this shortcoming, several approximations to this exact solution have been proposed; we discuss these and related ideas in \Cref{rel-wrk}. In this space, one high-level idea is to use the per-sample loss value as a proxy to the per-sample gradient norm, favoring samples with high losses for performing the update 
\citep{loshchilov2015online,shrivastava2016training,katharopoulos2017biased,kawaguchi2019ordered,zhang2019autoassist}. We discuss the important works based on this idea in detail in \Cref{rel-wrk}. Since exact loss computation is itself expensive especially for large models, some works rely on \textit{approximate} loss values for sample selection; for instance, using an auxiliary model to predict loss values of the actual model being trained \cite{katharopoulos2017biased,zhang2019autoassist}. So in summary, several approximations to the core idea of exact importance sampling have been proposed to make it practically usable.

Despite the myriad sample selection approaches, there is a paucity of theoretical results quantifying how much speed-up we can obtain over random sampling and their limitations -- especially for approaches involving the use of approximate quantities. In this work, we focus on the approach of choosing the sample with the highest loss for performing the gradient-based update. More specifically, suppose we are given $R > 1$ i.i.d. samples drawn from $\mathcal{D}$ and their \textbf{approximate} loss values with the constraint that we can pick only \textbf{one} sample by inspecting the \textit{approximate} losses to perform the update. We consider the \textit{greedy} approach of selecting the sample with the \textit{largest approximate loss}, and we call this \textbf{greedy SGD (GSGD)}. Also, we will refer to the default approach of picking a sample uniformly at random as \textit{(vanilla) SGD}. In \Cref{theory-sec}, we theoretically compare the convergence rates of GSGD (with \textit{approximate} losses) against SGD, characterizing the benefits and limitations of the former. We would like to clarify that \textit{we are not claiming that GSGD is a novel algorithm}; in fact, it is very similar to OSGD \cite{kawaguchi2019ordered} in spirit.\footnote{We do not call it OSGD because \citet{kawaguchi2019ordered} use exact loss values whereas we use approximate loss values. Moreover, \citet{kawaguchi2019ordered} consider the finite-sum (ERM) setting whereas we focus on the stochastic setting.} Rather, \textit{we are claiming that our theoretical characterization of its benefits and limitations -- with approximate losses -- is the first of its kind}, to the best of our knowledge. 

On the applied side, we propose to use \textit{early exiting} \citep{7900006,schwartz-etal-2020-right} as a light-weight way to obtain \textit{approximate} losses for sample selection in training large ML models. 
To be clear, \textit{our novelty is in the light-weight filtering process for training via early exiting}. 
{Empirical results on a 110M parameter BERT base model and a ResNet-50 model show that early exit-based sample selection yields significant improvements.}

We will now elaborate on our main {\bf contributions}.
\\
{\bf (a)} In \Cref{theory-sec}, we theoretically compare \textbf{greedy SGD (GSGD)} (i.e., pick the sample with the \textit{highest approximate loss}) against vanilla SGD (i.e., pick a sample at random). 
\begin{itemize}
    \item \Cref{thm-gsgd} provides a convergence bound for GSGD on smooth convex losses assuming that the argmax of the approximate losses is equal to the argmax of the actual losses (\Cref{asmp-order}). In \Cref{sec:approx-fn-val}, we consider a setting wherein \Cref{asmp-order} no longer holds. We quantify the degradation in the performance of GSGD in this setting; see Theorems \ref{thm-gsgd-2} and \ref{thm-approx-fn-val-jan21}.
    \item A \textbf{key insight} is that \textit{GSGD can converge to a (problem-dependent) constant factor of $F^{*} = \min_{\bm{w}} F(\bm{w})$ in fewer iterations than SGD}; see \Cref{main-insight}. This is of interest in training large ML models on extremely large datasets, where converging exactly to $F^{*}$ is infeasible and instead converging \textit{faster} to $\mathcal{O}(F^{*})$ is desirable. On the negative side, our result indicates that GSGD may not converge to $F^{*}$ asymptotically and hence it can be worse than SGD asymptotically (\Cref{gsgd-bad}).
\end{itemize}
\noindent {\bf (b)} 
In \Cref{early-exit}, we propose \textit{early exiting} as a way to cheaply obtain approximate losses with an intermediate layer's representations for sample selection in training large ML models such as \textit{transformers}. Specifically, we propose to do backpropagation on only 50\% of the samples in a batch with the \textit{highest approximate losses obtained via early exiting}. To our knowledge, early exiting has not been studied for accelerating \textit{training}. We call this early-exit based sample \enquote{sifting} process SIFT. This can be seamlessly integrated with other sample selection schemes; in particular, we also try large \textit{entropy}-based sample selection with early exiting.
\begin{itemize}
    \item In \Cref{sec:bert}, we show the efficacy of SIFT in training a 12 layer BERT base model with 110M parameters from scratch. 
    Specifically, to achieve 64\% validation accuracy, loss-based and entropy-based SIFT with exit at the first layer take roughly 43 and 40 hours, respectively, compared to roughly 57 hours of vanilla training involving no sample selection. 
    Further, in \Cref{resnet-50}, we show that SIFT is also very effective in speeding up the training of a modified ResNet-50 model that is amenable to early exiting. 
    \item In \Cref{early-exit-theory}, we quantify the probability of correctly selecting the sample with the largest \textit{actual loss} using early exiting for feedforward linear neural networks. 
\end{itemize}

\section{Related Work}
\label{rel-wrk}
There is a large body of work proposing several kinds of sample-selection schemes for accelerated training. These include importance sampling methods, sample reordering based approaches and algorithms focusing on samples with higher losses. Our approach falls in the last category.

{\bf Optimal importance sampling.} {Importance sampling} asks how should the training examples be sampled so as to obtain an unbiased stochastic gradient with the minimum possible variance. \citet{zhao2015stochastic, alain2015variance, needell2014stochastic} show that the optimal solution to this problem is to sample the examples with probability proportional to their per-sample gradient norms. Unfortunately, this is completely infeasible in practice because of the high cost of computing the per-sample gradient norms.

{\bf Approximate importance sampling.} Several papers attempt to approximate the optimal importance sampling procedure so as to make it feasible. For convex settings, \citet{zhao2015stochastic, needell2014stochastic} propose sampling with probability proportional to the smoothness constant of the per-sample losses, while \citet{borsos2018online,stich2017safe} propose adaptive sampling strategies. Going beyond convex settings, \citet{alain2015variance} present a distributed approach for importance sampling, while \citet{katharopoulos2017biased} approximate the importance weights with loss values which are predicted by a smaller network. 
\citet{katharopoulos2018not} derive an upper bound on the per-sample gradient norms for deep-learning networks which take essentially the same time to compute as the per-sample loss values, and propose using the upper bounds for approximating the importance weights. \citet{johnson2018training} approximate the true sampling distribution by solving a robust optimization problem. 

{\bf Sample reordering.} Another related line of work attempts to improve the way/order in which samples are presented while training. \citet{bengio2009curriculum} propose curriculum learning wherein the key idea is to present easier examples before harder ones but this requires prior information about the training set. Several improved modifications of this idea are out there \citep{tsvetkov2016learning,kumar2010self,jiang2017mentornet,kim2018screenernet,zhang2019autoassist,jiang2019accelerating}.
There is also a long line of papers that theoretically analyze conventional sample ordering schemes (such as shuffle once, random reshuffling, etc.) as well as improved sample ordering schemes \citep{bertsekas2011incremental,recht2012toward,gurbuzbalaban2019convergence,gurbuzbalaban2021convergence,haochen2019random,safran2020good,mishchenko2020random,lu2022general,mohtashami2022characterizing}. 

{\bf Selection of samples with large losses.} As mentioned in \Cref{sec:intro}, GSGD is \textit{not} a novel algorithm and there are prior approaches with the same underlying principle as GSGD, i.e., focus on the samples with the largest loss values.
The two algorithms closest to GSGD in spirit are online hard example mining (OHEM) proposed by \citet{shrivastava2016training} and ordered SGD (OSGD) proposed by \citet{kawaguchi2019ordered} -- \textit{except that both algorithms use exact losses} for sample selection. 
OHEM was proposed for training object detectors (in computer vision) and it involves back-propagating using the gradients of only the samples with the $k$ largest losses in batch of size $b$ with $k < b$. However, \citet{shrivastava2016training} do not provide any theoretical guarantees. OSGD \cite{kawaguchi2019ordered} is essentially the same algorithm as OHEM \cite{shrivastava2016training}, but with convergence guarantees and generalization bounds. However, \textit{unlike our theoretical results} for GSGD, \citet{kawaguchi2019ordered} do \textit{not} show how/when/to what extent ordered SGD is better than plain SGD. So even though GSGD is similar to OSGD in spirit, our theoretical results (for GSGD) are much more comprehensive.
\citet{loshchilov2015online} propose a sampling strategy which favors picking examples with larger losses. \citet{fan2017learning} introduce the average top-$k$ loss and advocate minimizing this loss rather than the empirical average over all the samples. 

{\bf Early Exiting.} Early exiting \citep{7900006,schwartz-etal-2020-right} is a promising approach to decreasing the computational cost of multilayered neural architectures by approximating the output of a model through its intermediate feature representations. This saves computational costs by dynamically deciding the number of layers/modules (attention-blocks in transformers) to use during inference by \textit{exiting} based on some metric computed on the intermediate representations themselves. Initially used for ResNets \citep{7900006}, early exiting is now widely popular even for transformer models, especially language models \citep{schwartz-etal-2020-right,xin-etal-2020-deebert,schuster2022confident, rotem2023finding,berxit,leebert}. Recent work around early exiting for large language models also explores accelerated decoding \citep{schuster2022confident} and improving factuality \citep{chuang2023dola}. In this work, we solely use early exiting as a light-weight way to obtain approximate losses with the intermediate layer representations for sample selection. To our knowledge, \textit{early exiting has not been used in prior work for speeding up training}.

\section{Notation}
\label{notation}
Vectors and matrices are in bold font. For a natural number $n$, we sometimes denote the set $\{1,\ldots,n\}$ by $[n]$. The (Gauss) error function and complementary error function are defined as:
\begin{equation}
    \label{erf}
    \text{erf}(t) := \frac{2}{\sqrt{\pi}}\int_{0}^t e^{-z^2} {d}z \text{ and } \text{erfc}(t) := 1 - \text{erf}(t).
\end{equation}
Note that $\lim_{t \to \infty} \text{erf}(t) = 1$. For any $z \in \mathbb{R}$, we define the sigmoid function as $\text{sig}(z) = \frac{1}{1 + e^z}$.

\section{Problem Setting}
\label{sec:oracle}
We briefly recap the optimization problem introduced in \Cref{sec:intro}. Given access to a first-order optimization oracle, we would like to minimize $F(\bm{w}) := \mathbb{E}_{\bm{x} \sim \mathcal{D}}[f(\bm{w}, \bm{x})]$, where $\bm{w} \in \mathbb{R}^d$ and $\mathcal{D}$ is the data distribution.

\textbf{Standard First-Order Stochastic Optimization Oracle.} A query at $\bm{w}$ returns $\nabla_{\bm{w}} f(\bm{w}, \bm{x})$, where $\bm{x} \sim \mathcal{D}$. 

We consider a more \textit{generous} oracle. However, it can also perform \textbf{one gradient evaluation}  per query (same as the standard oracle).
\begin{definition}[\textbf{Proposed First-Order Stochastic Optimization Oracle}]
    \label{new-oracle}
    A query at $\bm{w}$ first returns a set of $R > 1$ samples $\{\bm{x}^{(1)}, \ldots, \bm{x}^{(R)}\} = \mathcal{S}^{(R)}$ drawn i.i.d. from $\mathcal{D}$ and their \textit{approximate} function\footnote{Throughout this work, we interchange \enquote{loss} value and \enquote{function} value freely. In most cases, we use \enquote{function} value in the context of optimization-based discussions and \enquote{loss} value in the context of ML-based discussions.} values $\{\widetilde{f}(\bm{w}, \bm{x}^{(1)}), \ldots, \widetilde{f}(\bm{w}, \bm{x}^{(R)})\}$. The user can then pick \textbf{one} $\widehat{\bm{x}}$ from $\mathcal{S}^{(R)}$, and the oracle will return $\nabla_{\bm{w}} f(\bm{w}, \hat{\bm{x}})$.
\end{definition}
Later, we shall make an assumption on the relation between the approximate function value $\widetilde{f}(\bm{w}, \bm{x})$ and the actual function value ${f}(\bm{w}, \bm{x})$.
Now that we have introduced the oracle that we consider in this work, we state the $\widehat{\bm{x}}$ chosen by vanilla SGD and {greedy SGD (GSGD)}.

\textbf{SGD choice:} Pick $\widehat{\bm{x}}$ uniformly at random from $\mathcal{S}^{(R)}$. 

\textbf{GSGD choice:} Pick $\widehat{\bm{x}} = \text{arg max}_{\bm{x} \in \mathcal{S}^{(R)}} \widetilde{f}(\bm{w}, \bm{x})$.

We state the GSGD update rule in more detail next.

\subsection{Greedy SGD (GSGD) Algorithm}
In the $k^\text{th}$ iteration, we observe a set of $R$ i.i.d. samples drawn from $\mathcal{D}$, say $\mathcal{S}_k^{(R)} = \{\bm{x}_k^{(1)}, \ldots, \bm{x}_k^{(R)}\}$. We pick:
\begin{equation}
    \label{eq:gsgd-choice}
    \widehat{\bm{x}}_k = \text{arg max}_{\bm{x} \in \mathcal{S}_k^{(R)}} \widetilde{f}(\bm{w}_k, \bm{x}).
\end{equation}
The update of Greedy SGD (GSGD) with step-size $\eta_k$ is:
\begin{equation}
    \label{eq:gsgd-update}
    \bm{w}_{k+1} = \bm{w}_k - \eta_k \nabla f(\bm{w}_k, \widehat{\bm{x}}_k).
\end{equation}
The update of vanilla SGD is the same as \cref{eq:gsgd-update}, except with $\widehat{\bm{x}}_k$ being a random sample from $\mathcal{S}_k^{(R)}$.

\section{GSGD vs. SGD for Smooth Convex Objectives}
\label{theory-sec}
We begin by stating some assumptions and definitions.
\begin{assumption}[\textbf{Continuity}]
    \label{asmp-cont}
    For any $\bm{x} \sim \mathcal{D}$, $f(\bm{w}, \bm{x})$ is continuous w.r.t. $\bm{w}$.
\end{assumption}
\begin{assumption}[\textbf{Convexity}]
    \label{asmp-cvx}
    For any $\bm{x} \sim \mathcal{D}$, $f(\bm{w}, \bm{x})$ is convex w.r.t. $\bm{w}$.
\end{assumption}
\begin{assumption}[\textbf{Smoothness}]
    \label{asmp-smooth}
    For any $\bm{x} \sim \mathcal{D}$, $f(\bm{w}, \bm{x})$ is $L$-smooth w.r.t. $\bm{w}$.
\end{assumption}
\begin{assumption}
    \label{asmp-1}
    For any $\bm{x} \sim \mathcal{D}$, $\min_{\bm{w} \in \mathbb{R}^d} f(\bm{w}, \bm{x}) = 0$. 
\end{assumption}
Let $\Phi_F := \text{arg min}_{\bm{w} \in \mathbb{R}^d} F(\bm{w})$ and $F^{*} := \min_{\bm{w} \in \mathbb{R}^d} F(\bm{w})$. We restrict our attention to the case of $\Phi_F$ being closed and compact. 

Let us first consider the case where the argmax of the approximate function values is the same as the argmax of the exact function values; we will relax this assumption later in \Cref{sec:approx-fn-val}.
\begin{assumption}[\textbf{Approximate function values preserve argmax}]
    \label{asmp-order}
    In the setting of \Cref{new-oracle}, $\widetilde{f}$ satisfies:
    $$\text{arg max}_{\bm{x} \in \mathcal{S}^{(R)}} \widetilde{f}(\bm{w}, \bm{x}) = \text{arg max}_{\bm{x} \in \mathcal{S}^{(R)}} {f}(\bm{w}, \bm{x}).$$
\end{assumption}
Under \Cref{asmp-order}, $\widehat{\bm{x}}_k$ in \cref{eq:gsgd-choice} becomes the same as $\text{arg max}_{\bm{x} \in \mathcal{S}_k^{(R)}} {f}(\bm{w}_k, \bm{x})$. 
\begin{definition}
    \label{rho-def}
    For $R > 1$, let $$\widehat{F}_R(\bm{w}) := \mathbb{E}_{\{\bm{x}^{(j)}\}_{j=1}^R \text{ } \underset{\textup{iid}}{\sim} \text{ } \mathcal{D}}\Big[\max_{\text{ }\bm{x} \in \{\bm{x}^{(j)}\}_{j=1}^R}f(\bm{w}, \bm{x})\Big].$$ 
    Define $$\rho_R(\bm{w}) := \frac{\widehat{F}_R(\bm{w})}{F(\bm{w})} \text{ and } \rho_R^{*} := \inf_{\bm{w} \notin \Phi_F} \rho_R(\bm{w}).$$
    Also, suppose $$\sup_{\bm{w}^{*} \in \Phi_F}\widehat{F}_R(\bm{w}^{*}) \leq \Delta_R.$$
\end{definition}
Except for the trivial case of $f(\bm{w}, \bm{x}_1) = f(\bm{w}, \bm{x}_2)$ $\forall$ $\bm{x}_1, \bm{x}_2$ which we disregard, $\rho_R^{*}$ is \textit{strictly} bigger than 1. 

Further, consider a point $\widehat{\bm{w}}^{*} \notin \Phi_F$ that is $\epsilon$-close to some $\bm{w}^{*} \in \Phi_F$ (i.e., $\|\widehat{\bm{w}}^{*} - \bm{w}^{*}\|_2 \leq \epsilon$) and let $\epsilon \to 0$. In that case, under Assumption \ref{asmp-cont} (and because the max operation preserves continuity) and \Cref{rho-def}, $\widehat{F}_R(\widehat{\bm{w}}^{*}) \to \widehat{F}_R({\bm{w}}^{*}) \leq \Delta_R$. Also, $F(\widehat{\bm{w}}^{*}) \geq F^{*}$. Thus, by definition, $\rho_R^{*} \leq {\widehat{F}_R(\widehat{\bm{w}}^{*})}/{F(\widehat{\bm{w}}^{*})} \leq ({\Delta_R}/{F^{*}})$. Hence, we have that:
\begin{equation}
    \label{eq:rho}
    1 < \rho_R^{*} \leq \frac{\Delta_R}{F^{*}}, \text{ } \forall \text{ } R > 1.
\end{equation}
In \Cref{quantify-rho}, we quantify $\rho_R(\bm{w})$ and $\rho_R^{*}$ for fitting a model with the squared loss.

For our convergence results, we assume that our initialization is $\bm{w}_{0}$ and let $$D_0 := \min_{\bm{w}^{*} \in \Phi_F} \|\bm{w}_{0} - \bm{w}^{*}\|.$$ We are now ready to state our convergence results.
\begin{theorem}[\textbf{GSGD}]
    \label{thm-gsgd}
    Suppose Assumptions \ref{asmp-cont}, \ref{asmp-cvx}, \ref{asmp-smooth}, \ref{asmp-1} and \ref{asmp-order} hold. Set $\eta_k = \eta < \frac{1}{L}$ for all $k$. Then, GSGD has the following convergence guarantee after $K$ iterations:
    \begin{equation*}
        \mathbb{E}\Bigg[F\Bigg(\frac{1}{K} \sum_{k=0}^{K-1} \bm{w}_k\Bigg)\Bigg] \leq \frac{D_0^{2}}{2 \rho_R^{*} \eta \big(1 - \eta L\big) K} + \frac{\Delta_R}{\rho_R^{*} \big(1 - \eta L\big)}.
    \end{equation*}
\end{theorem}
The proof of \Cref{thm-gsgd} can be found in \Cref{gsgd-pf}.

We now state a corresponding folklore result for SGD.
\begin{theorem}[\textbf{SGD}]
    \label{thm-sgd}
    Suppose Assumptions \ref{asmp-cvx}, \ref{asmp-smooth} and \ref{asmp-1} hold. Set $\eta_k = \eta < \frac{1}{L}$ for all $k$. Then, SGD has the following convergence guarantee after $K$ iterations:
    \begin{equation*}
        \mathbb{E}\Bigg[F\Bigg(\frac{1}{K} \sum_{k=0}^{K-1} \bm{w}_k\Bigg)\Bigg] \leq \frac{D_0^{2}}{2 \eta \big(1 - \eta L\big) K} + \frac{\eta L F^{*}}{1 - \eta L} + F^{*}.
    \end{equation*}
\end{theorem}

From \cref{thm-gsgd}, observe that in general, GSGD \textit{may not} converge to the minimum value $F^{*}$ asymptotically (i.e., with $K \to \infty$). At best, we can show that GSGD converges to $\frac{\Delta_R}{\rho_R^{*}}$ which is $\geq F^{*}$ (this follows from \cref{eq:rho}). But SGD can indeed converge to $F^{*}$ asymptotically by setting $\eta = \frac{1}{\eta L \sqrt{K}}$ for example in \Cref{thm-sgd}. Based on this, we make the following remark.

\begin{remark}
    \label{gsgd-bad}
    GSGD may be worse than SGD asymptotically.
\end{remark}
However, \textbf{GSGD is better than SGD \textit{early on}} as $\rho_R^{*} > 1$ and assuming $\Delta_R = \mathcal{O}(F^{*})$ (and $F^{*} \neq 0$), \textbf{GSGD can converge to a constant factor of $F^{*}$ in fewer iterations than SGD}. We formalize this next.

\begin{corollary}[\textbf{Up to what point is GSGD better than SGD?}]
\label{cor-gsgd-vs-sgd}
Suppose we run GSGD and SGD with constant step-size $\eta < \frac{1}{L}$. In that case, until $K = \frac{D_0^{2} (\rho_R^{*} - 1)}{2 \eta (\Delta_R - \rho_R^{*} F^{*})}$
iterations, the convergence bound of GSGD in \Cref{thm-gsgd} is better than that of SGD in \Cref{thm-sgd}.

Thus, \textbf{GSGD can converge to $\frac{\Delta_R - F^{*}}{(1-\eta L) (\rho_R^{*} - 1)}$ function value {in fewer iterations} than SGD.}
\end{corollary}

\begin{remark}
    \label{main-insight}
    Based on \Cref{cor-gsgd-vs-sgd}, when $\Delta_R = \mathcal{O}(F^{*})$, GSGD can converge to $\mathcal{O}(F^{*})$ function value \textbf{in fewer iterations} than SGD. This is of particular interest in training large ML models such as transformers on extremely large datasets, \textit{where minimizing the training loss exactly is infeasible} and converging \textit{faster} to a constant factor of the minimum loss value is desirable.
\end{remark}

\subsection{Beyond Argmax-Preserving Approximate Function Values}
\label{sec:approx-fn-val}
Our previous results were under \Cref{asmp-order}, i.e., the argmax of the approximate function values ($\widetilde{f}$) is always equal to the argmax of the actual function values (${f}$). Here we relax this assumption by instead modeling $\widetilde{f}$ as a noisy version of $f$, and provide a convergence result for GSGD in such a setting. Modeling $\widetilde{f}$ as a noisy version of $f$ is analogous to modeling the stochastic gradients as a noisy version of the actual gradient in vanilla stochastic optimization. Specifically, we make the following assumption.
\begin{assumption}[\textbf{Approximate function values}]
    \label{approx-fn-val}
    There exists $\mu(\bm{w}) \in \mathbb{R}$ and $\sigma \geq 0$ such that:
    $$\widetilde{f}(\bm{w}, \bm{x}) = {f}(\bm{w}, \bm{x}) \exp\Big(\mu(\bm{w}) + \sigma {\zeta}(\bm{w}, \bm{x})\Big),$$ 
    where ${\zeta}(\bm{w}, \bm{x})$ is i.i.d. random noise with mean 0 and variance 1. 
\end{assumption}
Thus, the approximate function value is the actual function values times the exponential of random noise (i.e., $\exp\big(\sigma {\zeta}(\bm{w}, \bm{x})\big)$) times some other scaling (i.e., $\exp(\mu(\bm{w}))$). \Cref{approx-fn-val} is pretty mild as it does not involve any particular distributional assumptions on ${\zeta}(\bm{w}, \bm{x})$ (such as Gaussian, etc.).

The important thing to note is that under \Cref{approx-fn-val},  $\widehat{\bm{x}}_k$ in \cref{eq:gsgd-choice} is \textbf{not} always equal to $\text{arg max}_{\bm{x} \in \mathcal{S}_k^{(R)}} {f}(\bm{w}_k, \bm{x})$ -- unlike \Cref{asmp-order}. 

\begin{definition}
    \label{rho-def-2}
    Let 
    \small
    \begin{flalign*}
        & \widehat{F}_{R, \text{approx}}(\bm{w}) := 
        \\ 
        & \mathbb{E}_{\{\bm{x}^{(j)}\}_{j=1}^R \text{ } \underset{\textup{iid}}{\sim} \text{ } \mathcal{D}, \zeta}\Big[f\big(\bm{w}, \bm{x}^{(j^{*})}\big) \Big| j^{*} = \text{arg max}_{j \in [R]}\widetilde{f}(\bm{w}, \bm{x}^{(j)})\Big].
    \end{flalign*}
    \normalsize
    Define $$\rho_{R, \text{approx}}(\bm{w}) := \frac{\widehat{F}_{R, \text{approx}}(\bm{w})}{F(\bm{w})} \text{ and }$$  
    $$\rho_{R, \text{approx}}^{*} := \inf_{\bm{w} \notin \Phi_F} \rho_{R, \text{approx}}(\bm{w}).$$
\end{definition}
Clearly, $\widehat{F}_{R, \text{approx}}(\bm{w}) \leq \widehat{F}_{R}(\bm{w})$ (as defined in \Cref{rho-def}). So we also have:
$$\sup_{\bm{w}^{*} \in \Phi_F}\widehat{F}_{R, \text{approx}}(\bm{w}^{*}) \leq \sup_{\bm{w}^{*} \in \Phi_F}\widehat{F}_R(\bm{w}^{*}) \leq \Delta_R.$$
It is easy to extend the proof of \Cref{thm-gsgd} to obtain the following result for GSGD under \Cref{approx-fn-val}.
\begin{theorem}[\textbf{GSGD}]
    \label{thm-gsgd-2}
    Suppose Assumptions \ref{asmp-cont}, \ref{asmp-cvx}, \ref{asmp-smooth}, \ref{asmp-1} and \ref{approx-fn-val} hold. Set $\eta_k = \eta < \frac{1}{L}$ for all $k$. Then, GSGD has the following convergence guarantee after $K$ iterations:
    \begin{multline*}
        \mathbb{E}\Bigg[F\Bigg(\frac{1}{K} \sum_{k=0}^{K-1} \bm{w}_k\Bigg)\Bigg] \leq 
        \\
        \frac{D_0^{2}}{2 \rho_{R, \textup{approx}}^{*} \eta \big(1 - \eta L\big) K} 
        + \frac{\Delta_R}{\rho_{R, \textup{approx}}^{*} \big(1 - \eta L\big)}.
    \end{multline*}
\end{theorem}
The expectation in \Cref{thm-gsgd-2} also includes the randomness due to $\zeta(.)$ (i.e., the noise in the approximate function values). The proof of \Cref{thm-gsgd-2} is almost identical to the proof of \Cref{thm-gsgd} (see \Cref{gsgd-pf}) and is obtained by replacing $\widehat{F}_{R}(\bm{w})$ with $\widehat{F}_{R, \text{approx}}(\bm{w})$ and $\rho_{R}^{*}$ with $\rho_{R, \text{approx}}^{*}$.

For the subsequent results in this subsection, we shall consider $R=2$. Specifically, we will provide a lower bound for $\widehat{F}_{2, \text{approx}}(\bm{w}), \rho_{2, \text{approx}}(\bm{w})$ and $\rho_{2, \text{approx}}^{*}$ in terms of $\widehat{F}_{2}(\bm{w}), \rho_{2}(\bm{w})$ and $\rho_{2}^{*}$, respectively, as a function of the noise level $\sigma$.

\begin{theorem}
    \label{thm-approx-fn-val-jan21}
    Suppose \Cref{approx-fn-val} holds with $\sigma \leq \frac{1}{2\sqrt{2}}$. Define $p(\sigma) := \big(1 - 0.72 \big(1 - e^{-\sqrt{2} \sigma}\big)\big)$.
    Then:
    \begin{equation*}
        \widehat{F}_{2, \textup{approx}}(\bm{w}) \geq p(\sigma) \widehat{F}_2(\bm{w}),
    \end{equation*}
    \begin{equation*}
        \rho_{2, \textup{approx}}(\bm{w}) \geq p(\sigma) \rho_{2}(\bm{w}) \text{ and }
        \rho_{2, \textup{approx}}^{*} \geq p(\sigma) \rho_{2}^{*}.
    \end{equation*}
\end{theorem}
The proof of \Cref{thm-approx-fn-val-jan21} is in \Cref{thm-approx-fn-val-jan21-pf}. It is worth pointing out that our result is independent of $\mu(\bm{w})$; intuitively, this is because a constant scaling (w.r.t. the samples) does not change the argmax. Regarding the dependence w.r.t. $\sigma$, notice that $p(\sigma)$ is a decreasing function of $\sigma$. So as the noise level $\sigma$ increases, the lower bound becomes worse; this happens because the quality of the approximate function value worsens as $\sigma$ increases. Also, as a quick sanity check, observe that $p(0) = 1$. This makes sense because $\sigma = 0$ means no effective noise, i.e., the approximate function values match the actual function values modulo the constant scaling $\exp(\mu(\bm{w}))$.

In the following corollary, we specify a bound for the noise level $\sigma$ below which $\rho_{2, \textup{approx}}^{*} > 1$ and thus, we can obtain a speed-up over SGD.
\begin{corollary}
    As long as $\sigma < \sigma_\textup{max} := \frac{1}{\sqrt{2}} \log \big(\frac{18 \rho_{2}^{*}}{25 - 7 \rho_{2}^{*}}\big)$, $p(\sigma) \rho_{2}^{*} > 1$ and therefore, $\rho_{2, \textup{approx}}^{*} > 1$. Hence, for $\sigma < \sigma_\textup{max}$, GSGD with $R=2$ using approximate function values for sample selection can be faster than SGD.
\end{corollary}

\subsection{Quantifying $\rho_R(\bm{w})$}
\label{quantify-rho}
Here we shall quantify $\rho_R(\bm{w})$ for a particular case. Let us consider the problem of learning a parameterized model $\mathcal{M}(\bm{w}^{*}, .)$ with the squared loss. In this case, our per-sample objective function $f$ is: 
\begin{equation}
    f(\bm{w},\bm{x}) = \Big(\mathcal{M}(\bm{w}, \bm{x}) - \mathcal{M}(\bm{w}^{*}, \bm{x})\Big)^2.
\end{equation}
We make the following assumption.
\begin{assumption}
    \label{asmp-GP}
    For $\bm{w} \neq \bm{w}^{*}$, $\mathcal{M}(\bm{w}, \bm{x}) - \mathcal{M}(\bm{w}^{*}, \bm{x}) \underset{\text{iid}}{\sim} \mathcal{N}\big(\varepsilon(\bm{w}), \delta^2(\bm{w})\big)$, i.e., the per-sample  prediction error is a Gaussian random variable.
\end{assumption}
Modeling the per-sample prediction error as a Gaussian random variable is fairly reasonable and a similar assumption has been made in \cite{pennington2017geometry}. We provide a lower bound for $\rho(\bm{w})$ in this setting.
\begin{theorem}
    \label{rho-bound}
    Suppose \Cref{asmp-GP} holds. Let $\nu(R) := \sqrt{{\frac{\pi}{2} \log \frac{R}{4 \log R}}}$. Then:
    \begin{flalign*}
    & \rho_R(\bm{w}) \geq \\ 
    & \frac{\big(\varepsilon^2(\bm{w}) + \nu^2(R)\delta^2(\bm{w}) + 2 \nu(R) \varepsilon(\bm{w}) \delta(\bm{w})\big) \big(1 - \frac{1}{R}\big)}{\varepsilon^2(\bm{w}) + \delta^2(\bm{w})}.
    \end{flalign*}
\end{theorem}
The proof of \Cref{rho-bound} is in \Cref{rho-bound-pf}.
\begin{corollary}
    In the setting of \Cref{rho-bound}, if $\varepsilon(\bm{w}) \leq \mathcal{O}(\delta(\bm{w}))$ for all $\bm{w}$ in our region of optimization, then: 
    $$\rho_R^{*} \geq \Bigg(1 + \Omega\Bigg(\log \frac{R}{\log R} + \sqrt{\log \frac{R}{\log R}}\Bigg)\Bigg)\Big(1 - \frac{1}{R}\Big).$$
\end{corollary}

\section{Approximate Losses via Early Exiting in Neural Networks}
\label{early-exit}
{For sample selection in GSGD-like algorithms, we propose to use an approximation of the actual loss which is computed on the \enquote{early} predictions obtained by applying the linear head after the final layer to an intermediate layer's representation (instead of the final layer's representation). Specifically, suppose $\bm{\theta}$ is the linear head and for a sample $\bm{x}$ with label $y$, $\mathcal{R}(\bm{x})$ and $\widetilde{\mathcal{R}}(\bm{x})$ are the final layer's and some intermediate layer's representation with the same dimension as $\bm{\theta}$, respectively. Then the approximate and actual losses with the cross-entropy loss function denoted by $\ell$ are $\ell(y, \text{softmax}(\bm{\theta}^\top \widetilde{\mathcal{R}}(\bm{x})))$\footnote{For a vector $\bm{v} = [v_1,\ldots,v_d]$, $\text{softmax}(\bm{v}) = [\hat{v}_1,\ldots, \hat{v}_d]$ with $\hat{v}_i = {\exp({v_i})}/{\sum_{j=1}^d \exp({v_j})}$.} and $\ell(y, \text{softmax}(\bm{\theta}^\top {\mathcal{R}}(\bm{x})))$, respectively; we use the former to select samples for performing the gradient-based update.}

We shall now consider a feedforward linear neural network \citep{kawaguchi2016deep} and quantify the probability of correctly picking the sample with the largest actual loss when using early exiting.

\subsection{Probability of Correctly Selecting the Sample with the Largest Actual Loss}
\label{early-exit-theory}
We consider a binary classification problem, where each sample $\bm{x}$ has a binary label $y \in \{0,1\}$, which is a deterministic function of $\bm{x}$. Our model is a $k$-layer linear feed-forward network parameterized by 
$\{\bm{W}_i\}_{i=1}^d \in \mathbb{R}^{d \times d}$ and $\bm{\theta} \in \mathbb{R}^d$, where the soft prediction $\hat{y}$ for $\bm{x}$ is:
\begin{equation}
    \label{eq:7-2-jan30}
    \hat{y} = \text{sig}\Big(\bm{\theta}^\top \big(\bm{W}_k \times \bm{W}_{k-1} \times \ldots \times \bm{W}_1\big) \bm{x}\Big).
\end{equation}
In \cref{eq:7-2-jan30}, $\text{sig}(.)$ is the sigmoid function\footnote{For binary classification problems, the sigmoid function essentially plays the role of the softmax function mentioned earlier.} as defined in \Cref{notation}. For any $j \in [k]$, let us define:
\begin{equation}
    \bm{A}_j := \bm{W}_j \times \ldots \times \bm{W}_1 \text{ and } \bm{B}_j := \bm{W}_{k} \times \ldots \times \bm{W}_{j+1}.
\end{equation}
Then,
\begin{equation}
    \hat{y}_j = \text{sig}\big(\bm{\theta}^\top \bm{A}_j \bm{x}\big)
\end{equation}
is the \enquote{early prediction} for $\bm{x}$ at the $j^{\text{th}}$ layer. Note that $\hat{y} = \hat{y}_k$, where:
\begin{equation}
    \hat{y}_k = \text{sig}\big(\bm{\theta}^\top \bm{B}_j \bm{A}_j \bm{x}\big).
\end{equation}
In the context of GSGD, we once again focus on the case of $R=2$. For any two i.i.d. samples $\bm{x}^{(1)}$ and $\bm{x}^{(2)}$, let $y^{(1)}$ and $y^{(2)}$ $\in \{0,1\}$ be the corresponding ground truth labels and let $\hat{y}_j^{(1)}$ and $\hat{y}_j^{(2)}$ be the corresponding early predictions at the $j^{\text{th}}$ layer. Further, let $\ell_j^{(1)}$ and $\ell_j^{(2)}$ be the corresponding cross-entropy losses of the early predictions at the $j^{\text{th}}$ layer, i.e.,
\begin{equation}
    \ell_j^{(i)} = - y^{(i)}\log\big(\hat{y}_j^{(i)}\big) - (1-y^{(i)})\log\big(1-\hat{y}_j^{(i)}\big),
\end{equation}
for $i \in \{1,2\}$. In the context of GSGD, $\ell_j^{(1)}$ and $\ell_j^{(2)}$ are the approximate function values, whereas $\ell_k^{(1)}$ and $\ell_k^{(2)}$ are the actual function values. In this section, we are interested in quantifying the probability (over the randomness of data) that early exiting at the $j^\text{th}$ layer preserves the argmax; specifically, we wish to quantify
\begin{equation}
    p_{j} := \mathbb{P}_{\bm{x}^{(1)}, \bm{x}^{(2)}}\Big(\text{arg max}_{i \in [1,2]} \ell_j^{(i)} = \text{arg max}_{i \in [1,2]} \ell_k^{(i)}\Big).
\end{equation}
Note that $p_k = 1$. We will lower bound $p_{j}$ under the following distributional assumption on the data.
\begin{assumption}
    \label{asmp-gauss-early-exit}
    Let $\bar{y}^{(1)} = 2y^{(1)} - 1$ and $\bar{y}^{(2)} = 2y^{(2)} - 1$ be the centered ground truth labels ($\in \{-1,1\}$) of  $\bm{x}^{(1)}$ and $\bm{x}^{(2)}$, respectively. 
    Then, $\bar{y}^{(1)} \bm{x}^{(1)} - \bar{y}^{(2)} \bm{x}^{(2)} \sim \mathcal{N}(\vec{0}_d, 2 \bm{\textup{I}}_d)$, over the randomness of $\bm{x}^{(1)}$ and $\bm{x}^{(2)}$.\footnote{We only mention the randomness over $\bm{x}^{(1)}$ and $\bm{x}^{(2)}$ because $\bar{y}^{(1)}$ and $\bar{y}^{(2)}$ are deterministic functions of $\bm{x}^{(1)}$ and $\bm{x}^{(2)}$, respectively.} 
\end{assumption}
The Gaussianity assumption above has been made in order to obtain an exact expression for $p_j$ (in \Cref{thm-neural-net} below). \Cref{asmp-gauss-early-exit} can be potentially relaxed to obtain only lower bounds for  $p_j$. 
\begin{theorem}
    \label{thm-neural-net}
    Suppose \Cref{asmp-gauss-early-exit} holds. Define $$\beta_j := \frac{\langle \bm{A}_j^\top \bm{\theta}, \bm{A}_j^\top \bm{B}_j^\top \bm{\theta} \rangle}{\|\bm{A}_j^\top \bm{\theta}\|_2 \|\bm{A}_j^\top \bm{B}_j^\top \bm{\theta}\|_2}.$$
    We restrict our attention to the case of $\beta_j \geq 0$. Then:
    \begin{equation}
        \label{eq:13-jan30}
        p_j = 1 - \frac{1}{2 \sqrt{\pi}} \int_{0}^{\infty} \exp\Big(-\frac{y^2}{4}\Big) \textup{erfc}\Bigg(\frac{\beta_j y}{2 \sqrt{1-\beta_j^2}}\Bigg) dy,
    \end{equation}
    where $\textup{erfc}(.)$ is as defined in \cref{erf}. We also have the following lower bound for $p_j$:
    \begin{equation}
        \label{eq:14-jan30}
        p_j \geq 1 - \sqrt{\frac{2 - 2\beta_j^2}{2-\beta_j^2}}.
    \end{equation}
\end{theorem}
The proof of \Cref{thm-neural-net} can be found in \Cref{thm-neural-net-pf}. Note that $\beta_j$ is the (normalized) correlation between $\bm{A}_j^\top \bm{\theta}$ and $\bm{A}_j^\top \bm{B}_j^\top \bm{\theta}$. As per \Cref{thm-neural-net}, $p_j = 1$ when $\beta_j = 1$; this makes sense because $\beta_j = 1$ implies $\bm{A}_j^\top \bm{\theta}$ and $\bm{A}_j^\top \bm{B}_j^\top \bm{\theta}$ are parallel and so the argmax operation is unaffected if we use $\hat{y}_j$ instead of $\hat{y}_k$. It is worth mentioning that the lower bound for $p_j$ in \cref{eq:14-jan30} is loose for small $\beta_j$, but we believe it is tight up to constant factors for $\beta_j \approx 1$; in particular, it is exact in the case of $\beta_j = 1$. Specifically, we have the following simple corollary for $\beta_j \approx 1$. 
\begin{corollary}
    \label{cor-beta-feb1}
    In the setting of \Cref{thm-neural-net}, suppose $\beta_j = 1 - \tau_j$ where $\tau_j \to 0$. Then, $p_j \geq 1 - \mathcal{O}(\sqrt{\tau_j})$.
\end{corollary}
In simple words, if $\bm{A}_j^\top \bm{\theta}$ is strongly correlated with $\bm{A}_j^\top \bm{B}_j^\top \bm{\theta}$, then early exiting at the $j^{\text{th}}$ layer preserves the argmax with high probability.

In the next section, we empirically demonstrate the efficacy of early exiting in training a transformer model as well as a ResNet model.

\section{Empirical Evaluation}
\label{expts}
We demonstrate the efficacy of early exiting in accelerating the training of a BERT base model and a slightly modified version of ResNet-50 from scratch. Specifically, we apply early exiting for selecting samples in a mini-batch version of \textit{greedy} AdamW which is a simple extension of the greedy SGD idea to AdamW. We refer to this practical early exit-based sample selection or \enquote{sifting} strategy as \textit{SIFT}. We describe SIFT in more detail later but at a high level, we propose to backpropagate on only 50\% of the examples in a batch with the \textit{highest approximate losses obtained via early exiting}. We would like to emphasize that \textit{our novelty is in the light-weight sifting process for training via early exiting} and not the idea of backpropagating on the samples with large losses.\footnote{This is the same idea as OSGD \citep{kawaguchi2016deep} and OHEM \citep{shrivastava2016training}.}  

\subsection{BERT Base Model}
\label{sec:bert}
Here we consider the task of pretraining a BERT base model from scratch with the masked language modeling (MLM) loss. The BERT base model used in  our experiments consists of 12 layers with a hidden dimension of 768; \textit{the total number of parameters is 110M}. We train on BookCorpus \citep{zhu2015aligning} and English Wikipedia which are two diverse and extensive standard corpora. 
The validation set for assessing the model's performance is derived from the development partition of the training corpus. This partition ensures that the validation set represents a diverse and unbiased subset of the overall data. 
In our training, we use 512 as the maximum sequence length. To process the input data, we use the \enquote{bert-base-uncased} tokenizer from the Hugging Face model repository. 
The masking is applied after tokenization with a uniform masking rate of 15\%. Our experiments were conducted on AWS p4d.24xlarge instances (8 NVIDIA A100 Tensor core GPUs). 

Our baseline algorithm is the vanilla approach with no kind of sample filtering. The micro-batch per GPU is set to 32 sequences (so there are 32 sequences per GPU * 8 GPUs * 512 tokens = 131072 tokens per batch). 
We will now elaborate on SIFT which does greedy selection based on early exit loss. 

\textbf{Loss-based SIFT:} We select 50\% sequences per batch (for training) with the largest MLM losses  computed on the early predictions obtained from an intermediate layer in the way described in the beginning of \Cref{early-exit}. Specifically, we show results for the first, second, third, sixth and last (i.e., twelfth) layer. The micro-batch per GPU is set to 64 sequences. Since we select half of the sequences for training, the effective batch size per GPU is 32 which is the same as that of baseline.

The early-exit based filtering idea is pretty general in the sense that it can be seamlessly integrated with other sample selection schemes. In particular, we also tried greedy selection based on \textit{early exit entropy} (instead of loss). We describe it below. 

\textbf{Entropy-based SIFT:} Everything is the same as loss-based SIFT except that we select 50\% sequences per batch whose early predictions obtained from an intermediate layer have the largest entropies instead of MLM losses. Prediction entropy has been used a measure of model uncertainty for active learning; see for e.g., \citet{ren2021survey,gal2017deep}. So our proposal of entropy-based SIFT may also be of interest in active learning with large-scale models.

It is worth mentioning that in our implementation, the forward propagation for selecting samples in SIFT is implemented naively before full forward and backward propagation for the update step; one could potentially make this more efficient. However, even without any optimization of the initial forward propagation for sample selection, we obtain significant gains as we discuss later. 
For the experiments with SIFT, we do not perform any sample filtering for the first 20K steps; this is to provide an initial warm-up period and is aligned with the suggestion of \citet{kawaguchi2019ordered} for GSGD. The total number of update steps for each algorithm is 800K. We use the hyper-parameters suggested in the pretraining part of the original BERT paper \citep{devlin2018bert}; we defer these details to \Cref{expt-details}.

We would like to point out that the goal of our experiments is to only show the efficacy of early exiting for approximately selecting samples with large losses/entropies and \textit{not} propose a SOTA sample selection scheme.

\begin{figure}[t]
  \centering
  \includegraphics[width=0.5\textwidth]{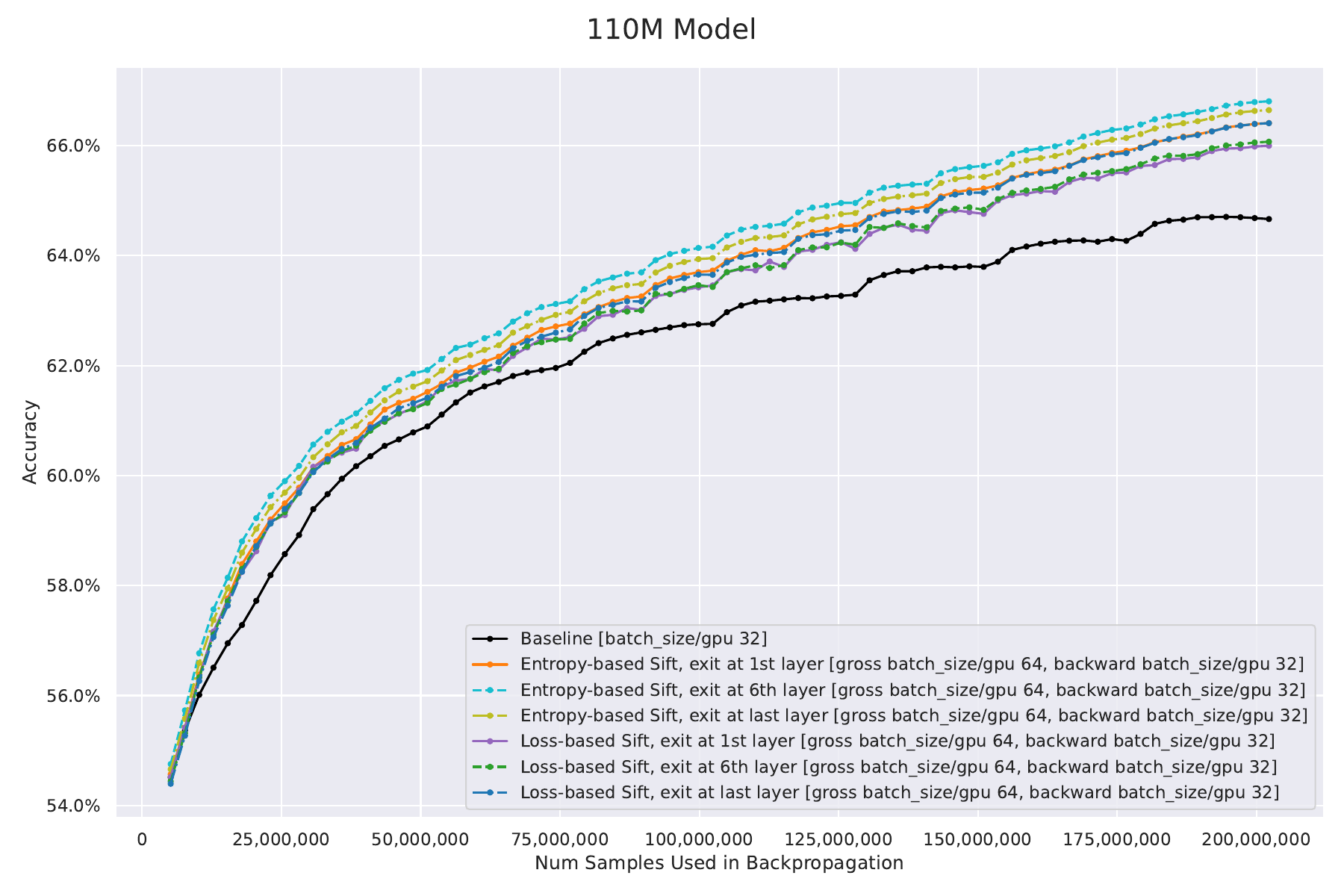}
  \caption{\textbf{Validation accuracy vs. backpropagation sample complexity.} In terms of performance, entropy-based SIFT $>$ loss-based SIFT $>$ baseline. All the accuracies are listed in \Cref{tab:table1b} but for quick reference, loss-based and entropy-based SIFT with exit at the first layer are better than the baseline by 1.33\% and 1.75\%, respectively. 
  For loss-based SIFT, exit at the last layer has the best performance while for entropy-based SIFT, exit at the sixth layer has the best performance.
  }
  \label{fig:perf_1}
\end{figure}

\begin{figure}[t]
  \centering
  \includegraphics[width=0.5\textwidth]{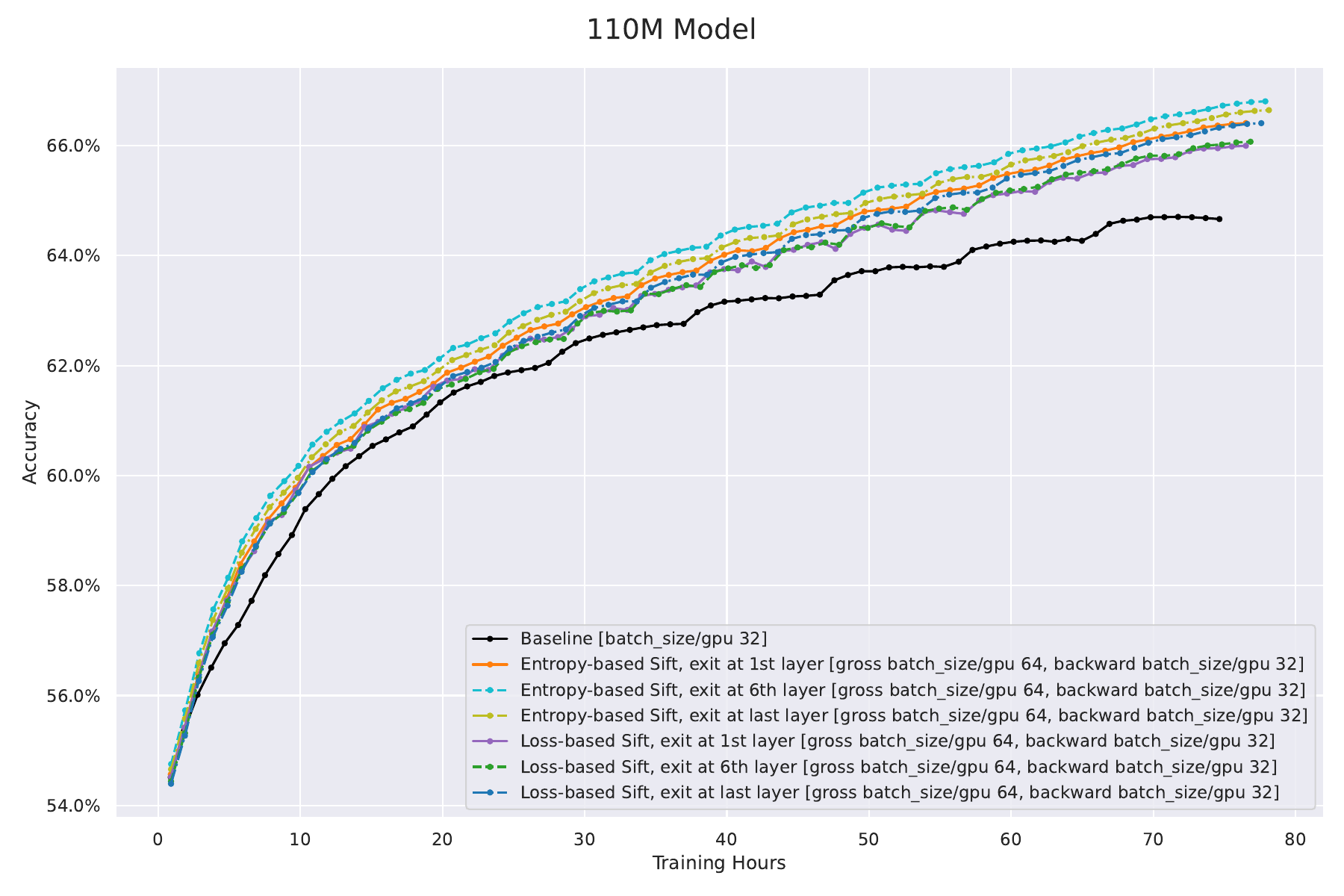}
  \caption{\textbf{Validation accuracy vs. wall-clock time.} The observations and trends are the same as \Cref{fig:perf_1}; please see the discussion therein.}
  \label{fig:perf_2}
\end{figure}

\begin{table}[!htbp]
\caption{\textbf{Validation accuracy at the last step.}
For loss-based SIFT, exit at the last layer has the best performance while the performance of other layers is nearly the same. However for entropy-based SIFT, exit at the sixth layer has the best performance.}
\vspace{0.2 cm}
\label{tab:table1b}
  \centering
\begin{tabular}{|c|c|c|}
  \hline
  Algorithm & SIFT Layer \# & Val. Accuracy  \\
  \hline
   Baseline         &  N/A & 0.6466\\
  \hline
   Loss-based SIFT        &  1 & 0.6599\\
   Loss-based SIFT         &  2 & 0.6597\\
   Loss-based SIFT        & 3 & 0.6604\\
   Loss-based SIFT          & 6 & 0.6607\\
   Loss-based SIFT          & 12 & 0.6640\\
   \hline
   Entropy-based SIFT          &  1 & 0.6641\\
   Entropy-based SIFT         & 2 & 0.6672\\
   Entropy-based SIFT          & 3 & 0.6678\\
   Entropy-based SIFT         & 6 & 0.6680\\
   Entropy-based SIFT       & 12 & 0.6664\\
   \hline   
\end{tabular}
\end{table}

\textbf{Results.} Figures \ref{fig:perf_1} and \ref{fig:perf_2} show the performance of SIFT with exits at the first, sixth and last layers\footnote{We do not show the plots for exits at the second and third layers here to avoid congestion; we report the validation accuracies of all these layers in \Cref{tab:table1b}.} and baseline w.r.t. the number of samples used for backpropagation (i.e., backpropagation sample complexity) and the number of training hours, respectively. We report the validation accuracy at the last step with early exit at all the layers we considered in \Cref{tab:table1b}. Please see the figure and table captions for detailed discussion but the key takeaways are as follows:
\begin{itemize}
    \item SIFT is better than the baseline in terms of backpropagation sample complexity as well as wall-clock time.
    \item Entropy-based SIFT does better than loss-based SIFT.
    \item For entropy-based SIFT, exit at the sixth layer has the best performance while for loss-based SIFT, exit at the last layer has the best performance.
\end{itemize}
In \Cref{tab:table2} (\Cref{expt-details}), we report the total time spent in the forward pass for selecting samples in SIFT as a function of the layer number.

\textit{Overall our results here show that SIFT can yield significant gains for BERT pretraining}.

\subsection{ResNet-50}
\label{resnet-50}
Here we consider training a slightly modified version of ResNet-50 on CIFAR-100 and Food-101 \cite{bossard2014food} which is a harder dataset than CIFAR-100 consisting of 101 classes. The modification has been made to make early exiting feasible in ResNet-50; we describe the modification and early exit details in \Cref{resnet-details}. Importantly, here we will \textit{tune the learning rates} for both SIFT and the baseline (i.e., the vanilla approach with no sample filtering). We use the one-cycle learning rate schedule proposed for \textit{fast training} in \cite{smith2019super} and available in PyTorch (\url{https://pytorch.org/docs/stable/generated/torch.optim.lr_scheduler.OneCycleLR.html}). Specifically, we tune the maximum learning rate (\enquote{max\textunderscore lr}) for the schedule and thus show results with different learning rates. We use default values for all other hyper-parameters of the schedule. We consider a scenario with limited training budget where we can train each algorithm for 100 epochs. Training is done with the standard cross-entropy loss function. Other empirical details are mentioned in \Cref{resnet-details}. 

Tables \ref{tabA} and \ref{tabC} show the comparisons of SIFT and baseline with the \textit{same training time} (analogous to \Cref{fig:perf_2}), whereas Tables \ref{tabB} and \ref{tabD} show comparisons with the same \textit{backpropagation sample complexity or number of gradient updates} (analogous to \Cref{fig:perf_1}) for CIFAR-100 and Food-101, respectively. In these experiments, we saw that loss-based sampling worked better than entropy-based sampling; so we only report the results for loss-based sampling here. 

Please refer to the table captions for detailed discussion but in summary, these results show that \textbf{SIFT is better even if we tune the learning rate}. Overall our results here show that SIFT can yield significant gains in training a ResNet-50 (appropriately modified to make early exiting feasible) from scratch. 
 
\begin{table}[!htbp]
\caption{\textbf{CIFAR-100 with same training time.} We run the baseline for 100 epochs and SIFT with early exit and SIFT with last layer exit \textit{till the time it takes to run 100 epochs of the baseline} (analogous to \Cref{fig:perf_2}) with different learning rates. The corresponding \textit{test accuracies} are reported. The best test accuracy of each method is in bold font. So with the \textit{same training time}, the \textbf{best test accuracy of SIFT with early exit $>$ best test accuracy of SIFT with last layer exit $>$ best test accuracy of baseline}.
}
  \label{tabA}
  \centering
    \begin{tabular}{|c|c|c|c|c|c|c|}
    \hline
    max\textunderscore lr & Baseline & \makecell{SIFT w/ \\ early exit} & \makecell{SIFT w/ last \\ layer exit} \\
    \hline
    $5e-2$ & $67.68$ & $66.64$ & $63.92$ \\
    $1e-2$ & $\bm{69.04}$ & $\bm{73.45}$ & ${70.13}$\\
    $5e-3$ & $69.02$ & $73.26$ & $\bm{72.91}$ \\
    $1e-3$ & $58.54$ & $66.85$ & $71.09$ \\
   \hline   
    \end{tabular}
\end{table}

\begin{table}[!htbp]
\caption{\textbf{CIFAR-100 with same backpropagation sample complexity.} We run all algorithms for 100 epochs so that the \textit{backpropagation sample complexity (i.e., \# of gradient updates) is the same for all algorithms} (analogous to \Cref{fig:perf_1}) with different learning rates. The corresponding \textit{test accuracies} are reported. The best test accuracy of each method is in bold font. So with the \textit{same gradient complexity}, the \textbf{best test accuracy of SIFT with last layer exit $>$ best test accuracy of SIFT with early exit $>$ best test accuracy of baseline}.
}
  \label{tabB}
  \centering
    \begin{tabular}{|c|c|c|c|c|c|c|}
    \hline
    max\textunderscore lr & Baseline & \makecell{SIFT w/ \\ early exit} & \makecell{SIFT w/ last \\ layer exit} \\
    \hline
    $5e-2$ & $67.68$ & $68.16$ & $68.12$ \\
    $1e-2$ & $\bm{69.04}$ & $\bm{75.41}$ & $74.48$\\
    $5e-3$ & $69.02$ & $75.06$ & $\bm{76.93}$\\
    $1e-3$ & $58.54$ & $68.35$ & $75.14$\\
   \hline   
    \end{tabular}
\end{table}

\begin{table}[!htbp]
\caption{\textbf{Food-101 with same training time.} All details are the same as in \Cref{tabA}. ‘--’ indicates non-convergence. Here, with the \textit{same training time}, the \textbf{best test accuracy of SIFT with last layer exit $>$ best test accuracy of SIFT with early exit $>$ best test accuracy of baseline}. 
}
  \label{tabC}
  \centering
    \begin{tabular}{|c|c|c|c|c|c|c|}
    \hline
    max\textunderscore lr & Baseline & \makecell{SIFT w/ \\ early exit} & \makecell{SIFT w/ last \\ layer exit} \\
    \hline
    $5e-2$ & $57.50$ & -- & -- \\
    $1e-2$ & $\bm{59.52}$ & $\bm{64.03}$ & $64.01$\\
    $5e-3$ & $58.88$ & $63.23$ & $\bm{66.47}$ \\
    $1e-3$ & $44.04$ & $59.29$ & $63.52$\\
   \hline   
    \end{tabular}
\end{table}

\begin{table}[!htbp]
\caption{\textbf{Food-101 with same backpropagation sample complexity.} All details are the same as in \Cref{tabB}. ‘--’ indicates non-convergence. In this case, with the \textit{same gradient complexity}, the \textbf{best test accuracy of SIFT with last layer exit $>$ best test accuracy of SIFT with early exit $>$ best test accuracy of baseline}.}
  \label{tabD}
  \centering
    \begin{tabular}{|c|c|c|c|c|c|c|}
    \hline
    max\textunderscore lr & Baseline & \makecell{SIFT w/ \\ early exit} & \makecell{SIFT w/ last \\ layer exit} \\
    \hline
    $5e-2$ & $57.50$ & -- & -- \\
    $1e-2$ & $\bm{59.52}$ & $\bm{64.04}$ & $64.19$\\
    $5e-3$ & $58.88$ & $63.22$ & $\bm{66.64}$ \\
    $1e-3$ & $44.04$ & $59.27$ & $63.68$\\
   \hline   
    \end{tabular}
\end{table}

\section{Conclusion and Limitations}
In this work, we theoretically characterized the benefits (as well as limitations) of the greedy approach of selecting samples with large approximate losses instead of exact losses. We also showed the promise of early exiting in speeding up the training of a transformer model. 

We will mention some limitations in our current work which we hope to explore and address in future work. As we mentioned, we did not pipeline the early exit forward propagation step for sample selection with the full forward and backward propagation steps for model update in our current implementation; doing so can yield bigger gains. In our current work, we have only shown the efficacy of SIFT on BERT and a modified version of ResNet. In the future, we hope to test SIFT on much larger transformer models. On the theory side, our current convergence result is for convex functions; we would like to derive a similar result for non-convex functions too.

\section*{Acknowledgements}
The authors are grateful to anonymous reviewers for their feedback which helped in improving this manuscript.

\section*{Impact Statement}
This paper presents work whose goal is to advance the field of machine learning. There are potential societal consequences of our work, none of which we feel must be specifically highlighted here.

\bibliography{refs,early_refs}
\bibliographystyle{icml2024}

\newpage
\appendix
\onecolumn

\begin{center}
    \textbf{\Large Appendix}\vspace{5mm}
\end{center}

\section{Some Preliminaries for the Proofs}
\begin{fact}[\citet{nesterov2018lectures}]
    \label{smoothness-bound}
    For an $L$-smooth function $g: \mathbb{R}^p \xrightarrow{} \mathbb{R}$, $\|\nabla g(\bm{y})\|_2^2 \leq 2 L (g(\bm{y}) - \min_{\bm{z} \in \mathbb{R}^p} g(\bm{z}))$ $\forall$ $\bm{y} \in \mathbb{R}^p$.
\end{fact}
\begin{fact}[\textbf{Cantelli's inequality}]
    \label{cantelli}
    For any random variable $Y$ and any $t > 0$:
    $$\mathbb{P}\Big(Y - \mathbb{E}[Y] \geq t\Big) \leq \frac{\textup{Var}(Y)}{\textup{Var}(Y) + t^2}.$$
\end{fact}
\begin{fact}[\citet{erf-bound-5}]
    \label{erf-bound}
    We have:
    $$\text{erf}(t) \leq \sqrt{1 - \exp\Big(-\frac{4 t^2}{\pi}\Big)}.$$
\end{fact}
The above bound has been also used by \citet{kamath_gauss_max}.
\begin{fact}
    \label{erf-bound-2}
    It holds that:
    $$\text{erfc}(t) \leq 2 \exp\Big(-\frac{t^2}{2}\Big).$$
\end{fact}
\begin{proof}
    Let $Z \sim \mathcal{N}(0,1)$. It can be checked that $\text{erf}(t) = 2 \mathbb{P}(Z \in (0,t))$. Thus,
    \begin{equation}
        \label{eq:7-jan30}
        \text{erfc}(t) = 1 - \text{erf}(t) = 2\Big(\frac{1}{2} - \mathbb{P}\big(Z \in (0,t)\big)\Big) = 2 \mathbb{P}(Z \geq t).
    \end{equation}
    But using the Gaussian tail bound, we have $\mathbb{P}(Z \geq t) \leq e^{-t^2/2}$. Using this in \cref{eq:7-jan30} gives us the desired result.
\end{proof}
\begin{fact}
    \label{fact5-jan30}
    For any $t > 0$:
    $$\int_{0}^\infty \exp\Big(-\frac{y^2}{2 t^2}\Big) dy = \sqrt{\frac{\pi}{2}} t.$$
\end{fact}
\begin{proof}
    Let $Z \sim \mathcal{N}(0,t^2)$. We have:
    \begin{equation}
        \int_{0}^\infty \exp\Big(-\frac{y^2}{2 t^2}\Big) dy = \sqrt{2 \pi t^2} \underbrace{\Bigg(\frac{1}{\sqrt{2 \pi t^2}} \int_{0}^\infty \exp\Big(-\frac{y^2}{2 t^2}\Big) dy\Bigg)}_{= \mathbb{P}(Z > 0) = \frac{1}{2}} = \sqrt{\frac{\pi}{2}} t. 
    \end{equation}
\end{proof}

\section{Proof of Theorem~\ref{thm-gsgd}}
\label{gsgd-pf}
\begin{proof}
Consider some $\bm{w}^{*} \in \Phi_F$. With a constant step-size, say $\eta$, we have for any iteration $k$:
\begin{equation}
    \label{eq:1}
    \mathbb{E}_{\mathcal{S}_k^{(R)}}\big[\|\bm{w}_{k+1} - \bm{w}^{*}\|^{2}\big] = \|\bm{w}_{k} - \bm{w}^{*}\|^{2} 
    - 2\eta \big \langle \mathbb{E}_{\mathcal{S}_k^{(R)}}\big[\nabla f(\bm{w}_k, \widehat{\bm{x}}_k)\big], \bm{w}_{k} - \bm{w}^{*} \big \rangle +
    \eta^{2}\mathbb{E}_{\mathcal{S}_k^{(R)}}\big[{\|\nabla f(\bm{w}_k, \widehat{\bm{x}}_k)\|^{2}}\big].
\end{equation}
Using \Cref{rho-def} in \cref{eq:1}, we get:
\begin{equation}
    \label{eq:1-jan15}
    \mathbb{E}_{\mathcal{S}_k^{(R)}}\big[\|\bm{w}_{k+1} - \bm{w}^{*}\|^{2}\big] = \|\bm{w}_{k} - \bm{w}^{*}\|^{2} 
    - 2\eta \langle \nabla \widehat{F}_R(\bm{w}_k), \bm{w}_{k} - \bm{w}^{*} \rangle +
    \eta^{2}\mathbb{E}_{\mathcal{S}_k^{(R)}}\big[{\|\nabla f(\bm{w}_k, \widehat{\bm{x}}_k)\|^{2}}\big].
\end{equation}
Using the convexity of $f(\bm{w}, \bm{x})$ w.r.t. $\bm{w}$ (\Cref{asmp-cvx}) and the fact that pointwise maximum of convex functions is also convex, we conclude that $\widehat{F}_R(\bm{w})$ is convex. Thus:
\begin{equation}
    \label{eq:2-new}
    \langle \nabla \widehat{F}_R(\bm{w}_k), \bm{w}_{k} - \bm{w}^{*} \rangle \geq \widehat{F}_R(\bm{w}_k) - \widehat{F}_R(\bm{w}^{*}) \geq \widehat{F}_R(\bm{w}_k) - \Delta_R,
\end{equation}
where the last step follows from the fact that $\sup_{\bm{w}^{*} \in \Phi_F}\widehat{F}_R(\bm{w}^{*}) \leq \Delta_R$.

Further, using \Cref{asmp-smooth}, \Cref{smoothness-bound} and \Cref{asmp-1}, we get:
\begin{equation}
    \label{eq:3-new}
    \mathbb{E}_{\mathcal{S}_k^{(R)}}\big[\|\nabla f(\bm{w}_k, \widehat{\bm{x}}_k)\|^{2}\big] \leq 2 L \mathbb{E}_{\mathcal{S}_k^{(R)}}\big[f(\bm{w}_k, \widehat{\bm{x}}_k)\big] \leq 2 L \widehat{F}_R(\bm{w}_k).
\end{equation}
Plugging \cref{eq:2-new} and \cref{eq:3-new} into \cref{eq:1}, we get:
\begin{equation}
    \label{eq:4-jan13}
    \mathbb{E}_{\mathcal{S}_k^{(R)}}\big[\|\bm{w}_{k+1} - \bm{w}^{*}\|^{2}\big] = \|\bm{w}_{k} - \bm{w}^{*}\|^{2} 
    - 2 \eta \big(1 - \eta L\big) \widehat{F}_R(\bm{w}_k)  + 2 \eta \Delta_R.
\end{equation}
Let us impose $\eta < \frac{1}{L}$ so that $1 - \eta L > 0$. From \Cref{rho-def}, we have $\widehat{F}_R(\bm{w}_k) \geq \rho_R^{*} F(\bm{w}_k)$. Using this above, we get:
\begin{equation}
    \label{eq:5-jan13}
    \mathbb{E}_{\mathcal{S}_k^{(R)}}\big[\|\bm{w}_{k+1} - \bm{w}^{*}\|^{2}\big] = \|\bm{w}_{k} - \bm{w}^{*}\|^{2}
    - 2 \rho_R^{*} \eta \big(1 - \eta L\big) F(\bm{w}_k) + 2 \eta \Delta_R.
\end{equation}
After some rearrangement, we get:
\begin{equation}
    F(\bm{w}_k) \leq \frac{\|\bm{w}_{k} - \bm{w}^{*}\|^{2} - \mathbb{E}_{\mathcal{S}_k^{(R)}}\big[\|\bm{w}_{k+1} - \bm{w}^{*}\|^{2}\big]}{2 \rho_R^{*} \eta \big(1 - \eta L\big)} + \frac{\Delta_R}{\rho_R^{*} \big(1 - \eta L\big)}.
\end{equation}
Next, we sum the above for all $k \in \{0,\ldots,K-1\}$ while taking expectation throughout. After dividing both sides of the resultant inequality by $\frac{1}{K}$, we get:
\begin{equation}
    \label{eq:9-jan13}
    \frac{1}{K}\sum_{k=0}^{K-1} \mathbb{E}[F(\bm{w}_k)] \leq \frac{\|\bm{w}_{0} - \bm{w}^{*}\|^{2}}{2 \rho_R^{*} \eta \big(1 - \eta L\big) K} + \frac{\Delta_R}{\rho_R^{*} \big(1 - \eta L\big)}.
\end{equation}
Let $\overline{\bm{w}}_K = \frac{1}{K} \sum_{k=0}^{K-1} \bm{w}_k$. Applying Jensen's inequality, we get $\mathbb{E}[F(\overline{\bm{w}}_K)] \leq \frac{1}{K}\sum_{k=0}^{K-1} \mathbb{E}[F(\bm{w}_k)]$. Using this in \cref{eq:9-jan13} gives us:
\begin{equation}
    \mathbb{E}\big[F\big(\overline{\bm{w}}_K\big)\big] \leq \frac{\|\bm{w}_{0} - \bm{w}^{*}\|^{2}}{2 \rho_R^{*} \eta \big(1 - \eta L\big) K} + \frac{\Delta_R}{\rho_R^{*} \big(1 - \eta L\big)}.
\end{equation}
Observe that this analysis holds for any $\bm{w}^{*} \in \Phi_F$, including the one that is closest to $\bm{w}_{0}$. Using this in the above bound gives us:
\begin{equation}
    \mathbb{E}\big[F\big(\overline{\bm{w}}_K\big)\big] \leq \frac{D_0^2}{2 \rho_R^{*} \eta \big(1 - \eta L\big) K} + \frac{\Delta_R}{\rho_R^{*} \big(1 - \eta L\big)},
\end{equation}
where $D_0 := \min_{\bm{w}^{*} \in \Phi_F} \|\bm{w}_{0} - \bm{w}^{*}\|$. 
\end{proof}

\section{Proof of Theorem~\ref{thm-approx-fn-val-jan21}}
\label{thm-approx-fn-val-jan21-pf}
\begin{proof}
    We have for $j \in \{1,2\}$, $\widetilde{f}(\bm{w}, \bm{x}^{(j)}) = {f}(\bm{w}, \bm{x}^{(j)}) \exp\Big(\mu(\bm{w}) + \sigma {\zeta}(\bm{w}, \bm{x}^{(j)})\Big)$, where ${\zeta}(\bm{w}, \bm{x}^{(j)})$ is i.i.d. random noise with mean 0 and variance 1. 
    For brevity of notation, let:
    $$\mu = \mu(\bm{w}) \text{ and } \widetilde{f}^{(j)} = \widetilde{f}(\bm{w}, \bm{x}^{(j)}), \text{ } f^{(j)} = {f}(\bm{w}, \bm{x}^{(j)}) \text{ and } {\zeta}^{(j)} = {\zeta}(\bm{w}, \bm{x}^{(j)}) \text{ for } j \in \{1,2\}.$$
    Without loss of generality, let $f^{(1)} \geq f^{(2)}$. 

    Let $j^{*} = \text{arg max}_{j \in [2]} \widetilde{f}^{(j)}$. Let us consider the case of $j^{*} = 1$. This happens when $f^{(2)} \exp(\mu + \sigma {\zeta}^{(2)}) \leq f^{(1)} \exp(\mu + \sigma {\zeta}^{(1)})$; this is equivalent to:
    \begin{equation}
        \sigma ({\zeta}^{(2)} - {\zeta}^{(1)}) \leq \log \frac{f^{(1)}}{f^{(2)}}.
    \end{equation}
    Let $Z = {\zeta}^{(2)} - {\zeta}^{(1)}$. Note that $Z$ is random variable with mean 0 and variance 2. As per the above discussion, we have:
    \begin{equation}
        \mathbb{P}\big(j^{*} = 1\big) = \mathbb{P}\Bigg(Z \leq \frac{1}{\sigma} \log \frac{f^{(1)}}{f^{(2)}}\Bigg) = 1 - \mathbb{P}\Bigg(Z > \frac{1}{\sigma} \log \frac{f^{(1)}}{f^{(2)}}\Bigg).
    \end{equation}
    Thus,
    \begin{equation}
        \mathbb{P}\big(j^{*} = 2\big) = \mathbb{P}\Bigg(Z > \frac{1}{\sigma} \log \frac{f^{(1)}}{f^{(2)}}\Bigg).
    \end{equation}
    Then, we have:
    \begin{flalign}
        \mathbb{E}_{\{\zeta^{(1)}, \zeta^{(2)}\}}\Big[f^{(j^{*})}\Big] 
        & = f^{(1)} \mathbb{P}\big(j^{*} = 1\big) + f^{(2)} \mathbb{P}\big(j^{*} = 2\big)
        \\
        & = f^{(1)} \Bigg(1 - \mathbb{P}\Bigg(Z > \frac{1}{\sigma} \log \frac{f^{(1)}}{f^{(2)}}\Bigg)\Bigg) + f^{(2)} \mathbb{P}\Bigg(Z > \frac{1}{\sigma} \log \frac{f^{(1)}}{f^{(2)}}\Bigg)
        \\
        \label{eq:30-jan26}
        & = f^{(1)} - \big(f^{(1)} - f^{(2)}\big) \mathbb{P}\Bigg(Z > \frac{1}{\sigma} \log \frac{f^{(1)}}{f^{(2)}}\Bigg).
    \end{flalign}
    Since $\mathbb{E}[Z] = 0$ and $\text{Var}(Z) = 2$, using Cantelli's inequality (\Cref{cantelli}), we have:
    \begin{equation}
        \mathbb{P}\Bigg(Z > \frac{1}{\sigma} \log \frac{f^{(1)}}{f^{(2)}}\Bigg) \leq \frac{2}{2 + \frac{1}{\sigma^2} \log^2 \frac{f^{(1)}}{f^{(2)}}}.
    \end{equation}
    Using this in \cref{eq:30-jan26}, we get:
    \begin{equation}
        \label{eq:31-jan21}
        \mathbb{E}_{\{\zeta^{(1)}, \zeta^{(2)}\}}\Big[f^{(j^{*})}\Big] \geq f^{(1)} - \big(f^{(1)} - f^{(2)}\big) \Bigg(\frac{2}{2 + \frac{1}{\sigma^2} \log^2 \frac{f^{(1)}}{f^{(2)}}}\Bigg).
    \end{equation}
    Let us define the function $h(t; \sigma) := \frac{2(1 - e^{-t})}{2 + \frac{t^2}{\sigma^2}}$ for $t \geq 0$. Then, note that: $$\big(f^{(1)} - f^{(2)}\big) \Bigg(\frac{2}{2 + \frac{1}{\sigma^2} \log^2 \frac{f^{(1)}}{f^{(2)}}}\Bigg) = f^{(1)} \times h\Bigg(\log \frac{f^{(1)}}{f^{(2)}}; \sigma\Bigg).$$ 
    From \Cref{lem-1-jan21}, we have $h(t; \sigma) \leq 0.72 \big(1 - e^{-\sqrt{2} \sigma}\big)$ for all $t \geq 0$, when $\sigma \leq \frac{1}{2\sqrt{2}}$. Using this above gives us:
    \begin{equation}
        \big(f^{(1)} - f^{(2)}\big) \Bigg(\frac{2}{2 + \frac{1}{\sigma^2} \log^2 \frac{f^{(1)}}{f^{(2)}}}\Bigg) \leq 0.72 \Big(1 - e^{-\sqrt{2} \sigma}\Big) f^{(1)},
    \end{equation}
    when $\sigma \leq \frac{1}{2\sqrt{2}}$.
    Plugging this into \cref{eq:31-jan21} gives us:
    \begin{equation}
        \mathbb{E}_{\{\zeta^{(1)}, \zeta^{(2)}\}}\Big[f^{(j^{*})}\Big] \geq \Big(1 - 0.72 \Big(1 - e^{-\sqrt{2} \sigma}\Big)\Big) f^{(1)}.
    \end{equation}
    Rewriting the above in terms of the full notation, we obtain:
    \begin{equation}
        \mathbb{E}_{\{{\zeta}(\bm{w}, \bm{x}^{(j)})\}_{j=1}^2}\Big[f\big(\bm{w}, \bm{x}^{(j^{*})}\big) \Big| j^{*} = \text{arg max}_{j \in [2]}\widetilde{f}(\bm{w}, \bm{x}^{(j)})\Big] \geq \Big(1 - 0.72 \Big(1 - e^{-\sqrt{2} \sigma}\Big)\Big) \max_{j \in [2]} {f}(\bm{w}, \bm{x}^{(j)}).
    \end{equation}
    Thus:
    \begin{equation}
        \widehat{F}_{2, \text{approx}}(\bm{w}) \geq \Big(1 - 0.72 \Big(1 - e^{-\sqrt{2} \sigma}\Big)\Big) \mathbb{E}_{\{\bm{x}^{(j)}\}_{j=1}^2}\Big[\max_{j \in [2]} {f}(\bm{w}, \bm{x}^{(j)})\Big] = \Big(1 - 0.72 \Big(1 - e^{-\sqrt{2} \sigma}\Big)\Big) \widehat{F}_2(\bm{w}).
    \end{equation}
    Hence, we also have:
    \begin{equation}
        \rho_{2, \text{approx}}(\bm{w}) = \frac{\widehat{F}_{2, \text{approx}}(\bm{w})}{F(\bm{w})} \geq \Big(1 - 0.72 \Big(1 - e^{-\sqrt{2} \sigma}\Big)\Big) \frac{\widehat{F}_2(\bm{w})}{F(\bm{w})} = \Big(1 - 0.72 \Big(1 - e^{-\sqrt{2} \sigma}\Big)\Big) \rho_{2}(\bm{w}),
    \end{equation}
    and:
    \begin{equation}
        \rho_{2, \text{approx}}^{*} = \inf_{\bm{w} \notin \Phi_F} \rho_{2, \text{approx}}(\bm{w}) \geq \Big(1 - 0.72 \Big(1 - e^{-\sqrt{2} \sigma}\Big)\Big) \inf_{\bm{w} \notin \Phi_F} \rho_{2}(\bm{w}) = \Big(1 - 0.72 \Big(1 - e^{-\sqrt{2} \sigma}\Big)\Big) \rho_{2}^{*}.
    \end{equation}
    
\end{proof}

\begin{lemma}
    \label{lem-1-jan21}
    Consider the function $h(t; \sigma) := \frac{2(1 - e^{-t})}{\big(2 + \frac{t^2}{\sigma^2}\big)}$ for $t \geq 0$ and $\sigma \leq \frac{1}{2\sqrt{2}}$. Then, we have:
    \begin{equation*}
        \max_{t \geq 0} h(t; \sigma) \leq 0.72 \Big(1 - e^{-\sqrt{2} \sigma}\Big).
    \end{equation*}
\end{lemma}

\begin{proof}
Let $t^{*} = \text{arg max}_{t \geq 0} h(t; \sigma)$. Setting $\frac{{d} h}{{d} t}\big|_{t=t^{*}} = 0$, we obtain the following equation: 
\begin{equation}
    e^{t^{*}} = 1 + \frac{\sigma^2}{t^{*}} + \frac{t^{*}}{2}.
\end{equation}    
Using the series expansion of the exponential function above, we get:
\begin{equation}
    \label{eq:34-jan21}
    1 + \sum_{j=1}^\infty \frac{(t^{*})^j}{j!} = 1 + \frac{\sigma^2}{t^{*}} + \frac{t^{*}}{2} \implies (t^{*})^2 \underbrace{\Big(1 + 2 \sum_{j=1}^\infty \frac{(t^{*})^j}{(j+1)!}\Big)}_{:=\nu(t^{*})} = 2 \sigma^2.
\end{equation}
Since $\nu(t^{*}) \geq 1$, we conclude that:
\begin{equation}
    \label{eq:35-jan21}
    t^{*} \leq \sqrt{2} \sigma.
\end{equation}
Further, since $\sigma \leq \frac{1}{2\sqrt{2}}$, we also have $t^{*} \leq \frac{1}{2}$. Note that $\nu(t^{*})$ is an increasing function of $t^{*}$. So $\nu(t^{*}) \leq \nu\big(\frac{1}{2}\big)$. Also, by using the series expansion of the exponential function, it can be verified that $\nu(t^{*}) = \frac{2(e^{t^{*}} - 1)}{t^{*}} - 1$. Thus, we have:
\begin{equation}
    \nu(t^{*}) \leq \nu\Big(\frac{1}{2}\Big) \leq 4 e^{1/2} - 5.
\end{equation}
Using this in \cref{eq:34-jan21}, we get:
\begin{equation}
    \label{eq:37-jan21}
    t^{*} \geq \frac{\sqrt{2}\sigma}{\sqrt{4 e^{1/2} - 5}} \geq 0.62 \big(\sqrt{2} \sigma\big).
\end{equation}
Combining the bounds of \cref{eq:35-jan21} and \cref{eq:37-jan21}, we deduce that $t^{*} = c \big(\sqrt{2} \sigma\big)$, where $c \in [0.62, 1]$. Thus, $(1 - e^{-t^{*}}) \leq 1 - e^{-\sqrt{2} \sigma}$ and $\frac{2}{2 + \frac{t^2}{\sigma^2}} \leq \frac{2}{2(1 + 0.62^2)} \leq 0.72$. Using all of this, we obtain:
\begin{equation}
    \max_{t \geq 0} h(t; \sigma) = h(t^{*}; \sigma) \leq 0.72 \big(1 - e^{-\sqrt{2} \sigma}\big).
\end{equation}
\end{proof}

\section{Proof of Theorem~\ref{rho-bound}}
\label{rho-bound-pf}
We restate \Cref{rho-bound} before proving it.
\begin{theorem}
    \label{rho-bound-app}
    Suppose \Cref{asmp-GP} holds. Then:
    \begin{flalign*}
    \rho_R(\bm{w}) \geq \frac{\Big(\varepsilon^2(\bm{w}) + \Big({\frac{\pi}{2} \log \frac{R}{4 \log R}}\Big)\delta^2(\bm{w}) + \sqrt{2 \pi \log \frac{R}{4 \log R}} \varepsilon(\bm{w}) \delta(\bm{w})\Big) \Big(1 - \frac{1}{R}\Big)}{\varepsilon^2(\bm{w}) + \delta^2(\bm{w})}.
    \end{flalign*}
\end{theorem}
\begin{proof}
Per \Cref{asmp-GP}, $\mathcal{M}(\bm{w}, \bm{x}) - \mathcal{M}(\bm{w}^{*}, \bm{x}) \underset{\text{iid}}{\sim} \mathcal{N}\big(\varepsilon(\bm{w}), \delta^2(\bm{w})\big)$ for $\bm{w} \neq \bm{w}^{*}$. 
Clearly,
\begin{equation}
    F(\bm{w}) = \mathbb{E}_{\bm{x}}\big[f(\bm{w},\bm{x})\big] = \varepsilon^2(\bm{w}) + \delta^2(\bm{w}),
\end{equation}
and
\begin{equation}
    \widehat{F}_R(\bm{w}) = \mathbb{E}_{\{\bm{x}^{(1)}, \ldots, \bm{x}^{(R)}\}}\Bigg[\max_{\bm{x} \in \{\bm{x}^{(1)}, \ldots, \bm{x}^{(R)}\}} f(\bm{w},\bm{x})\Bigg].
\end{equation}
For conciseness, we shall denote $\varepsilon(\bm{w})$ and $\delta(\bm{w})$ by just $\varepsilon$ and $\delta$, respectively. Also, let $Z_i = \mathcal{M}(\bm{w}, \bm{x}^{(i)}) - \mathcal{M}(\bm{w}^{*}, \bm{x}^{(i)})$ for $i \in [R]$. As per our concise notation, note that each $Z_i \underset{\text{iid}}{\sim} \mathcal{N}(\varepsilon, \delta^2)$. Hence:
\begin{equation}
    \label{eq:40-jan22}
    \widehat{F}_R(\bm{w}) = \mathbb{E}\Big[\max_{i \in [R]} Z_i^2\Big].
\end{equation}
We shall obtain a lower bound for $\mathbb{E}\big[\max_{i \in [R]} Z_i^2\big]$. Let $Y \sim \mathcal{N}(0,1)$. For $t > 0$, we have:
\begin{flalign}
    \mathbb{P}\Big(\max_{i \in [R]} (Z_i - \varepsilon) \leq t\Big) = \mathbb{P}\Big(\cap_{i \in [R]} (Z_i - \varepsilon) \leq t\Big) & = \Big(\mathbb{P}\big(Y \leq {t}/{\delta}\big)\Big)^R
    \\
    \label{eq:43-jan22}
    & = \Big(\frac{1}{2} + \frac{1}{2}\text{erf}\Big(\frac{t}{\sqrt{2}\delta}\Big)\Big)^R
    \\
    \label{eq:44-jan22}
    & \leq \Bigg(\frac{1}{2} + \frac{1}{2}\sqrt{1 - \exp\Big(-\frac{2 t^2}{\pi \delta^2}\Big)}\Bigg)^R
    \\
    \label{eq:45-jan22}
    & \leq \Bigg(\frac{1}{2} + \frac{1}{2}\Bigg(1 - \frac{1}{2}\exp\Big(-\frac{2 t^2}{\pi \delta^2}\Big)\Bigg)\Bigg)^R
    \\
    & = \Bigg(1 - \frac{1}{4}\exp\Big(-\frac{2 t^2}{\pi \delta^2}\Big)\Bigg)^R
    \\
    \label{eq:47-jan22}
    & \leq \exp\Big(-\frac{R}{4}\exp\Big(-\frac{2 t^2}{\pi \delta^2}\Big)\Big).
\end{flalign}
In \cref{eq:43-jan22}, the error function is as defined in \cref{erf}. \Cref{eq:44-jan22} follows from \Cref{erf-bound}. \Cref{eq:45-jan22} is obtained by using the fact that $\sqrt{1 - a} \leq 1 - \frac{a}{2}$ for all $a \in [0,1]$. \Cref{eq:47-jan22} follows from the fact that $1 - a \leq e^{-a}$ for all $a \in \mathbb{R}$. So, $\max_{i \in [R]} (Z_i - \varepsilon) \leq t$ with a probability of at most $\exp\Big(-\frac{R}{4}\exp\Big(-\frac{2 t^2}{\pi \delta^2}\Big)\Big)$. Let us choose $t = \delta \sqrt{\frac{\pi}{2} \log \frac{R}{4 \log R}}$. With this choice, we get:
\begin{equation*}
    \max_{i \in [R]} Z_i \geq \varepsilon + \delta \sqrt{\frac{\pi}{2} \log \frac{R}{4 \log R}} \text{ } \text{ w.p. } \geq 1 - \frac{1}{R}.
\end{equation*}
Thus,
\begin{flalign}
    \mathbb{E}\Big[\max_{i \in [R]} Z_i^2\Big] \geq \Bigg(\varepsilon + \delta \sqrt{\frac{\pi}{2} \log \frac{R}{4 \log R}}\Bigg)^2 \Big(1 - \frac{1}{R}\Big) = \Bigg(\varepsilon^2 + {\frac{\pi}{2} \log \frac{R}{4 \log R}}\delta^2 + \sqrt{2 \pi \log \frac{R}{4 \log R}} \varepsilon \delta\Bigg) \Big(1 - \frac{1}{R}\Big).
\end{flalign}
Plugging this into \cref{eq:40-jan22} and replacing $\varepsilon$ and $\delta$ with their complete notations, i.e., $\varepsilon(\bm{w})$ and $\delta(\bm{w})$, we get:
\begin{equation}
    \widehat{F}_R(\bm{w}) \geq \Bigg(\varepsilon^2(\bm{w}) + \Big({\frac{\pi}{2} \log \frac{R}{4 \log R}}\Big)\delta^2(\bm{w}) + \sqrt{2 \pi \log \frac{R}{4 \log R}} \varepsilon(\bm{w}) \delta(\bm{w})\Bigg) \Big(1 - \frac{1}{R}\Big).
\end{equation}
Hence:
\begin{equation}
    \rho_R(\bm{w}) = \frac{\widehat{F}_R(\bm{w})}{F(\bm{w})} \geq \frac{\Big(\varepsilon^2(\bm{w}) + \Big({\frac{\pi}{2} \log \frac{R}{4 \log R}}\Big)\delta^2(\bm{w}) + \sqrt{2 \pi \log \frac{R}{4 \log R}} \varepsilon(\bm{w}) \delta(\bm{w})\Big) \Big(1 - \frac{1}{R}\Big)}{\varepsilon^2(\bm{w}) + \delta^2(\bm{w})}.
\end{equation}
\end{proof}

\section{Proof of Theorem~\ref{thm-neural-net}}
\label{thm-neural-net-pf}
\begin{proof}
For $i \in \{1,2\}$, we have:
\begin{equation}
    \hat{y}_j^{(i)} = \text{sig}\big(\bm{\theta}^\top \bm{A}_j \bm{x}^{(i)}\big) \text{ and } \hat{y}^{(i)} = \hat{y}_k^{(i)} = \text{sig}\big(\bm{\theta}^\top \bm{B}_j \bm{A}_j \bm{x}^{(i)}\big).
\end{equation}
For i.i.d. samples $\bm{x}^{(1)}$ and $\bm{x}^{(2)}$ with labels $y^{(1)}$ and $y^{(2)}$, recall that $\ell_j^{(1)}$ and $\ell_j^{(2)}$ are the corresponding cross-entropy losses of the early predictions at the $j^{\text{th}}$ layer, i.e.,
\begin{equation}
    \ell_j^{(i)} = - y^{(i)}\log\big(\hat{y}_j^{(i)}\big) - (1-y^{(i)})\log\big(1-\hat{y}_j^{(i)}\big) \text{ for } i \in \{1,2\}.
\end{equation}
We are interested in:
\begin{equation*}
    p_{j} := \mathbb{P}_{\bm{x}^{(1)}, \bm{x}^{(2)}}\Big(\text{arg max}_{i \in [1,2]} \ell_j^{(i)} = \text{arg max}_{i \in [1,2]} \ell_k^{(i)}\Big).
\end{equation*}
Using symmetry, we get:
\begin{equation}
    \label{eq:59-jan30}
    p_{j} := \mathbb{P}_{\bm{x}^{(1)}, \bm{x}^{(2)}}\Big(\ell_j^{(1)} \geq \ell_j^{(2)} \Big| \ell_k^{(1)} \geq \ell_k^{(2)}\Big).
\end{equation}
Note that $p_k = 1$. Henceforth, we shall drop the subscript $\bm{x}^{(1)}, \bm{x}^{(2)}$ for conciseness. 

As per our notation in \Cref{asmp-gauss-early-exit}, recall that $\bar{y}^{(1)} = 2y^{(1)} - 1$ and $\bar{y}^{(2)} = 2y^{(2)} - 1$ are the centered labels of $\bm{x}^{(1)}$ and $\bm{x}^{(2)}$, respectively. It can be verified that:
\begin{equation}
    \label{eq:53-jan30}
    \ell_j^{(1)} \geq \ell_j^{(2)} \iff  \bm{\theta}^\top \bm{A}_j \Big(\bar{y}^{(1)} \bm{x}^{(1)} - \bar{y}^{(2)} \bm{x}^{(2)}\Big) \geq 0 \text{ and } \ell_k^{(1)} \geq \ell_k^{(2)} \iff  \bm{\theta}^\top \bm{B}_j \bm{A}_j \Big(\bar{y}^{(1)} \bm{x}^{(1)} - \bar{y}^{(2)} \bm{x}^{(2)}\Big) \geq 0.
\end{equation}
Let $\bm{z} := \bar{y}^{(1)} \bm{x}^{(1)} - \bar{y}^{(2)} \bm{x}^{(2)}$. As per \Cref{asmp-gauss-early-exit},  $\bm{z} \sim \mathcal{N}(\vec{0}_d, 2 \bm{\text{I}}_d)$. 

Using \cref{eq:53-jan30} and the definition of $\bm{z}$ followed by the application of Bayes' theorem in \Cref{eq:59-jan30}, we obtain:
\begin{equation}
    p_j = \mathbb{P}\Big(\bm{\theta}^\top \bm{A}_j \bm{z} \geq 0 \Big| \bm{\theta}^\top \bm{B}_j \bm{A}_j \bm{z} \geq 0 \Big) = \frac{\mathbb{P}\big(\bm{\theta}^\top \bm{A}_j \bm{z} \geq 0,  \bm{\theta}^\top \bm{B}_j \bm{A}_j \bm{z} \geq 0 \big)}{\mathbb{P}\big(\bm{\theta}^\top \bm{B}_j \bm{A}_j \bm{z} \geq 0\big)}.
\end{equation}
Since $\bm{z} \sim \mathcal{N}(\vec{0}_d, 2 \bm{\text{I}}_d)$, $\mathbb{P}\big(\bm{\theta}^\top \bm{B}_j \bm{A}_j \bm{z} \geq 0\big) = \frac{1}{2}$. Using this above, we get:
\begin{equation}
    \label{eq:56-jan30}
    p_j = 2 \mathbb{P}\Big(\bm{\theta}^\top \bm{A}_j \bm{z} \geq 0, \bm{\theta}^\top \bm{B}_j \bm{A}_j \bm{z} \geq 0 \Big).
\end{equation}
For ease of notation, let $u_1 = \bm{\theta}^\top \bm{A}_j \bm{z}$ and $u_2 = \bm{\theta}^\top \bm{B}_j \bm{A}_j \bm{z}$ and $\bm{u} = \begin{bmatrix}
u_1 \\
u_2
\end{bmatrix}$. Note $\bm{u}$ is a multivariate Gaussian random variable with $\mathbb{E}[\bm{u}] = \begin{bmatrix}
0 \\
0
\end{bmatrix}$ and
\begin{equation}
    \mathbb{E}[\bm{u}\bm{u}^\top] := \bm{\Sigma} =  2 \begin{bmatrix}
\|\bm{A}_j^\top \bm{\theta}\|_2^2 & \big \langle \bm{A}_j^\top \bm{\theta}, \bm{A}_j^\top \bm{B}_j^\top \bm{\theta} \big \rangle \\
\big \langle \bm{A}_j^\top \bm{\theta}, \bm{A}_j^\top \bm{B}_j^\top \bm{\theta} \big \rangle & \|\bm{A}_j^\top \bm{B}_j^\top \bm{\theta}\|_2^2
\end{bmatrix}
.
\end{equation}
Let $\alpha_1 = \|\bm{A}_j^\top \bm{\theta}\|_2$, $\alpha_2 = \|\bm{A}_j^\top \bm{B}_j^\top \bm{\theta}\|_2$ and $$\beta = \frac{\langle \bm{A}_j^\top \bm{\theta}, \bm{A}_j^\top \bm{B}_j^\top \bm{\theta} \rangle}{\|\bm{A}_j^\top \bm{\theta}\|_2 \|\bm{A}_j^\top \bm{B}_j^\top \bm{\theta}\|_2}.$$ In the theorem statement, we shall denote $\beta$ by $\beta_j$ to indicate the dependence on the layer number $j$; we drop the subscript $j$ here for conciseness. With this notation, we have:
\begin{equation}
    \bm{\Sigma} =  2 \begin{bmatrix}
\alpha_1^2 & \beta \alpha_1 \alpha_2 \\
\beta \alpha_1 \alpha_2 & \alpha_2^2
\end{bmatrix}
.
\end{equation}
The density function of $\bm{u}$ is therefore:
\begin{flalign}
    \phi(\bm{u}) & = \frac{1}{2 \pi \sqrt{\text{det}(\bm{\Sigma})}} \exp\Bigg(-\frac{1}{2} \bm{u}^\top \bm{\Sigma}^{-1} \bm{u}\Bigg) 
    \\
    & = \frac{1}{4 \pi \alpha_1 \alpha_2 \sqrt{1 - \beta^2}} \exp\Bigg(-\frac{1}{4(1-\beta^2)}\Big(\frac{u_1^2}{\alpha_1^2} - \frac{2\beta u_1 u_2}{\alpha_1 \alpha_2} + \frac{u_2^2}{\alpha_2^2}\Big)\Bigg).
\end{flalign}
Using this in \cref{eq:56-jan30}, we get:
\begin{flalign}
    p_j & = 2 \int_{0}^{\infty} \int_{0}^{\infty} \phi(\bm{u}) {d} u_1 {d} u_2
    \\
    & = \frac{1}{2 \pi \alpha_1 \alpha_2 \sqrt{1 - \beta^2}} \int_{0}^{\infty} \int_{0}^{\infty} 
    \exp\Bigg(-\frac{1}{4(1-\beta^2)}\Big(\frac{u_1^2}{\alpha_1^2} - \frac{2\beta u_1 u_2}{\alpha_1 \alpha_2} + \frac{u_2^2}{\alpha_2^2}\Big)\Bigg) {d} u_1 {d} u_2
    \\
    & = \frac{1}{2 \pi \alpha_1 \alpha_2 \sqrt{1 - \beta^2}} \int_{0}^{\infty} \exp\Big(-\frac{u_1^2}{4 \alpha_1^2}\Big) \Bigg(\int_{0}^{\infty} \exp\Bigg(-\frac{\big(u_2 - \big(\frac{\beta \alpha_2}{\alpha_1}\big) u_1\big)^2}{4 \alpha_2^2 (1-\beta^2)}\Bigg) d u_2 \Bigg) d u_1.
\end{flalign}
With some simple change of variables and some simplification, the above equation becomes:
\begin{flalign}
    p_j & = \frac{1}{\pi \alpha_1} \int_{0}^{\infty} \exp\Big(-\frac{u_1^2}{4 \alpha_1^2}\Big) \Bigg(\int_{-\frac{\beta u_1}{2 \sqrt{1-\beta^2} \alpha_1}}^\infty \exp(-t^2) dt \Bigg) du_1
    \\
    & = \frac{1}{2 \sqrt{\pi} \alpha_1} \int_{0}^{\infty} \exp\Big(-\frac{u_1^2}{4 \alpha_1^2}\Big) \Bigg(\frac{2}{\sqrt{\pi}}\int_{0}^{\frac{\beta u_1}{2 \sqrt{1-\beta^2} \alpha_1}} \exp(-t^2) dt  + \frac{2}{\sqrt{\pi}}\int_{0}^\infty \exp(-t^2) dt\Bigg) du_1
    \\
    & = \frac{1}{2 \sqrt{\pi} \alpha_1} \int_{0}^{\infty} \exp\Big(-\frac{u_1^2}{4 \alpha_1^2}\Big) \Bigg(\text{erf}\Big(\frac{\beta u_1}{2 \sqrt{1-\beta^2} \alpha_1}\Big) + \underbrace{\lim_{y \to \infty} \text{erf}(y)}_{=1}\Bigg) du_1
    \\
    & = \frac{1}{2 \sqrt{\pi} \alpha_1} \int_{0}^{\infty} \exp\Big(-\frac{u_1^2}{4 \alpha_1^2}\Big) \text{erf}\Big(\frac{\beta u_1}{2 \sqrt{1-\beta^2} \alpha_1}\Big) du_1 + \underbrace{\frac{1}{2 \sqrt{\pi} \alpha_1} \int_{0}^{\infty} \exp\Big(-\frac{u_1^2}{4 \alpha_1^2}\Big) du_1}_{=\frac{1}{2} \text{ using \Cref{fact5-jan30}}}
    \\
    & = \frac{1}{2 \sqrt{\pi} \alpha_1} \int_{0}^{\infty} \exp\Big(-\frac{u_1^2}{4 \alpha_1^2}\Big) \text{erf}\Big(\frac{\beta u_1}{2 \sqrt{1-\beta^2} \alpha_1}\Big) du_1 + \frac{1}{2}.
\end{flalign}
Replacing $({u_1}/{\alpha_1})$ by $y$ above gives us:
\begin{flalign}
    p_j & = \frac{1}{2 \sqrt{\pi}} \int_{0}^{\infty} \exp\Big(-\frac{y^2}{4}\Big) \text{erf}\Big(\frac{\beta y}{2 \sqrt{1-\beta^2}}\Big) dy + \frac{1}{2}
    \\
    & = \underbrace{\frac{1}{2 \sqrt{\pi}} \int_{0}^{\infty} \exp\Big(-\frac{y^2}{4}\Big) dy}_{=\frac{1}{2} \text{ using \Cref{fact5-jan30}}} - \frac{1}{2 \sqrt{\pi}} \int_{0}^{\infty} \exp\Big(-\frac{y^2}{4}\Big) \text{erfc}\Bigg(\frac{\beta y}{2 \sqrt{1-\beta^2}}\Bigg) dy + \frac{1}{2}
    \\
    \label{eq:78-jan30}
    & = 1 - \frac{1}{2 \sqrt{\pi}} \int_{0}^{\infty} \exp\Big(-\frac{y^2}{4}\Big) \text{erfc}\Bigg(\frac{\beta y}{2 \sqrt{1-\beta^2}}\Bigg) dy.
\end{flalign}
This is the exact final expression for $p_j$. 

Using \Cref{erf-bound-2} fact above yields:
\begin{flalign}
    p_j & \geq 1 - \frac{1}{\sqrt{\pi}} \int_{0}^{\infty} \exp\Big(-\frac{y^2}{4}\Big) \exp\Big(-\frac{\beta^2 y^2}{8 {(1-\beta^2)}}\Big) dy 
    \\
    & = 1 - \frac{1}{\sqrt{\pi}} \int_{0}^{\infty} \exp\Bigg(-\frac{y^2}{8} \Big(\frac{2 - \beta^2}{1 - \beta^2}\Big)\Bigg) dy 
    \\
    \label{eq:81-jan30}
    & = 1 - \sqrt{\frac{2 - 2\beta^2}{2-\beta^2}}.
\end{flalign}
\Cref{eq:81-jan30} follows from \Cref{fact5-jan30}. 

Finally, the theorem statement follows by adding the subscript $j$ to $\beta$ in \Cref{eq:78-jan30} and \Cref{eq:81-jan30} (recall that we omitted the subscript $j$ earlier for conciseness).
\end{proof}

\subsection{Proof of Corollary~\ref{cor-beta-feb1}}
\begin{proof}
    Plugging in $\beta_j = 1 - \tau_j$ into the lower bound for $p_j$ in \Cref{thm-neural-net}, we get:
    \begin{flalign*}
        p_j \geq 1 - \sqrt{\frac{2 - 2(1 - \tau_j)^2}{2-(1 - \tau_j)^2}} = 1 - \sqrt{\frac{2 \tau_j (2-\tau_j)}{1 + \tau_j(2-\tau_j)}} = 1 - \mathcal{O}(\sqrt{\tau_j}),
    \end{flalign*}
    for $\tau_j \to 0$.
\end{proof}

\section{Remaining Experimental Details for Section~\ref{sec:bert}}
\label{expt-details}
For AdamW, we used the following hyper-parameter values: learning rate = 1e-4, $\ell_2$ weight decay = 0.01, $\beta_1$ = 0.9 and $\beta_2$ = 0.999. The learning rate warmup was over the first 0.2\% of total steps followed by linear decay. We used the GELU activation and a dropout probability of 0.1 on all the layers. The training loss is the sum of the mean masked LM likelihood and the mean next sentence prediction likelihood.

\begin{table*}[htbp]
\caption{Sampling time of SIFT as a function of early exit layer. For comparison, the full forward and backward propagation times (for the update) in all the cases is $\sim$ 3.8 hours and $\sim$ 7.9 hours, respectively.}
\vspace{0.2 cm}
  \label{tab:table2}
  \centering
    \begin{tabular}{|c|c|c|c|c|c|c|}
    \hline
    SIFT Criterion & SIFT Layer \# & Sampling Time (hrs) \\
    \hline
    Loss-based         & 1 & 0.8452\\
    Loss-based          & 2 & 1.1662\\
    Loss-based        & 3 & 1.4764 \\
    Loss-based           & 6 & 2.5901\\
    Loss-based          & 12 & 4.5086\\
    \hline
    Entropy-based           & 1 & 0.9121 \\ 
    Entropy-based         & 2 & 1.2181\\
    Entropy-based          & 3 & 1.5398\\
    Entropy-based         & 6 & 2.6195\\
    Entropy-based         & 12 & 4.5148\\
   \hline   
    \end{tabular}
\end{table*}

\section{Remaining Details for Section~\ref{resnet-50}}
\label{resnet-details}
\textbf{Modified ResNet architecture:} The vanilla ResNet architecture is not really amenable to the early exit idea we used for BERT. This is because, unlike BERT, the intermediate layer representations of vanilla ResNet are not of the same size as the final layer representations due to which we cannot use the linear classifier at the head for computing \enquote{early loss/entropy} like we did for BERT. So we slightly modify the vanilla ResNet architecture as follows: the output of our modified architecture = linear classifier at head ($L_f$) $\times$ final layer’s representation ($R_f$) + another linear classifier of appropriate size ($L_i$) $\times$ some intermediate layer’s representation ($R_i$) instead of just $L_f \times R_f$ which is the output of vanilla ResNet; note that $L_i$ is also trained. To compute the early loss/entropy, we use $L_i \times R_i$. For our experiments in \Cref{resnet-50}, we modify the vanilla ResNet-50 architecture as described above and the intermediate layer we use is the final output of the second block consisting of 128 filters. The modification described here for ResNets does indeed use an additional linear classifier ($L_i$) increasing the number of parameters (compared to the vanilla architecture), but we think it is worthwhile given the amount of improvement we get with SIFT. Moreover, the number of extra parameters is not too much for datasets like CIFAR-100, Food-101, etc., wherein the number of classes is of the order of 100. 

\textbf{Other empirical details:} Here, the baseline batch size is 125. The gross batch size of SIFT is 250, while the forward/backward batch size of SIFT is 125 (i.e., we perform the SIFT update on top 50\% of the samples in a batch just like \Cref{sec:bert}). We use the default values of $\beta_1$ and $\beta_2$ for AdamW and set the weight decay $=5e-4$.

\end{document}